\documentclass[3p]{elsarticle}
\usepackage{float}
\usepackage{graphicx}
\usepackage{color,soul}
\usepackage{amsmath,amssymb,verbatim,bm}
\usepackage{ctable} 
\usepackage{url}
\usepackage{bm}
\usepackage{algorithm}
\usepackage{algpseudocode}
\usepackage[english]{babel}
\usepackage{ulem} 
\usepackage{mdframed}
\newmdenv[
  leftmargin = 0pt,
  innerleftmargin = 1em,
  innertopmargin = 0pt,
  innerbottommargin = 0pt,
  innerrightmargin = 0pt,
  rightmargin = 0pt,
  linewidth = 3pt,
  topline = false,
  rightline = false,
  bottomline = false,
  linecolor=red,
]{leftbar}

\definecolor{marine}{RGB}{85,100,255}
\definecolor{pomegrenate}{RGB}{255,18,94}

\usepackage{xspace}
\usepackage{xparse}
\usepackage[bookmarks,bookmarksopen,bookmarksnumbered,colorlinks,pdfpagemode = {UseOutlines},pdfstartview = {FitH},urlcolor = {blue}]{hyperref}
\usepackage{amsmath,amssymb,verbatim,bm}
\usepackage{float}
\usepackage{graphicx}
\usepackage{algorithm}
\usepackage{algpseudocode}
\DeclareRobustCommand{\qed}{$\square$}
\DeclareMathOperator{\theargmin}{argmin}
\DeclareMathOperator{\BERN}{Bernoulli}
\DeclareMathOperator{\variance}{var}
\newcommand{\SBM}{{\mathop{\scriptscriptstyle \mathrm{SBM}}}}

\newcommand{\Gk}{G^{(k)}}

\DeclareMathOperator{\M}{\mathbb{M}}
\DeclareMathOperator{\N}{\mathbb{N}}
\DeclareMathOperator{\R}{\mathbb{R}}

\DeclareMathOperator{\cM}{\mathcal{M}}
\DeclareMathOperator{\cO}{\mathcal{O}}
\DeclareMathOperator{\cI}{\mathcal{I}}

\newcommand{\bx}{\bm{x}}
\newcommand{\bA}{\bm{A}}

\newcommand{\cG}{\mathcal{G}}

\newcommand{\bQ}{\bm{Q}}

\newcommand{\bp}{\bm{p}}
\newcommand{\bs}{\bm{s}}
\newcommand{\blamb}{\bm{\lambda}}

\DeclareRobustCommand{\argmin}[1]{\underset{#1}{\theargmin}\mspace{4mu}}
\DeclareRobustCommand{\bern}[1]{\BERN \left(#1\right)}
\DeclareRobustCommand{\var}[1]{\variance\left[#1\right]}

\NewDocumentCommand \E { m o}
{
  \IfNoValueTF {#2} {\mathbb{E}\left[#1\right]}{\mathbb{E}_{#2}\mspace{-4mu}\left[#1\right] }
}
\NewDocumentCommand \prob { m o}
{
  \IfNoValueTF {#2} {\mathbb{P}\left(#1\right)}{\mathbb{P}_{#2} \left(#1\right)}
}

\newcommand{\ER}{Erd\H{o}s-R\'enyi}

\makeatletter
\def\@opargbegintheorem#1#2#3{\trivlist
\item[]{\bfseries #1\ #2\ (#3)} \itshape}
\makeatother

\newtheorem{corollary}{Corollary}
\newtheorem{Definition}{Definition}

\newtheorem{Lemma}{Lemma}

\newtheorem{lproof}{Proof of Lemma}
\newtheorem{proof}{Proof of Theorem}
\newtheorem{Remark}{Remark}
\newtheorem{theorem}{Theorem}

\begin{document}
\begin{frontmatter}

  \title{Theoretical analysis and computation of the sample Fr\'echet mean for sets of large graphs based on spectral information}
  \author{Daniel Ferguson}
  \author{Fran\c{c}ois G. Meyer\fnref{fnt2}}
  \address{Applied Mathematics, University of Colorado at Boulder, Boulder CO 80305}
  \fntext[fnt2]{Corresponding author: \href{mailto:fmeyer@colorado.edu}{\sf fmeyer@colorado.edu}\hfill \\
    \hspace{0.5cm}This work was supported by the National Science Foundation, 
    \href{https://www.nsf.gov/awardsearch/showAward?AWD_ID=1815971}{CCF/CIF 1815971}.}

  \begin{abstract}
    \small\baselineskip=9pt To characterize the location (mean, median) of a set of graphs, one needs a notion of centrality that is adapted to metric spaces, since graph sets are not Euclidean spaces. A standard approach is to consider the Fr\'echet mean. In this work, we equip a set of graphs with the pseudometric  defined by the $\ell_2$ norm between the eigenvalues of their respective adjacency matrix. Unlike the edit distance, this pseudometric reveals structural changes at multiple scales, and is well adapted to studying various statistical problems for graph-valued data. We describe an algorithm to compute an approximation to the sample Fr\'echet mean of a set of undirected unweighted graphs with a fixed size using this pseudometric.
  \end{abstract}

\end{frontmatter}

\section{Introduction.}
Machine learning almost always requires the estimation of the average of a dataset. Algorithms for clustering, classification, and linear regression all utilize the average value \cite{HTF08}. When the distance is induced by a norm, then the mean is a simple algebraic operation. If the data lie on a Riemannian manifold, equipped with a metric, then one can extend the notion of mean with the concept of Fr\'echet mean \cite {P06}. In fact the concept of Fr\'echet mean only requires that a (pseudo)metric between points be defined, and therefore one can consider the Fr\'echet mean of a set in a pseudometric space \cite{F48}.  Not surprisingly, many machine learning algorithms, which  were developed for Euclidean spaces, can be extended to use the Fr\'echet mean. The purpose of this paper is to solve the nontrivial problem of determining the Fr\'echet mean for data sets of graphs when the pseudometric is the $\ell_2$ distance between the eigenvalues of the adjacency matrices of two graphs.

In this work we consider a set of simple graphs with $n$ vertices which have an edge density, $\rho_n$, that satisfies,
\begin{equation}
  n^{-2/3} \ll \rho_n \ll 1.
\end{equation}
We note that the vertex set must be sufficiently large and that the technique introduced in this paper will perform poorly for sets of small graphs.

Our line of attack involves the following two intermediate results: (1) a graph's largest eigenvalues can be approximated within any precision by the corresponding eigenvalues of a realization of a stochastic block model; (2) given a set of target eigenvalues, one can recover the stochastic block model whose Fr\'echet mean has a spectrum that best matches these target eigenvalues. We prove various error bounds and convergence results for our algorithm and validate the theory with several experiments. 

\section{State of the Art.}
We consider the set of undirected, unweighted graphs of fixed size $n$, wherein we define a distance. To characterize the location (mean, median) of the set  of graphs we need a notion of centrality that is adapted to metric spaces, since graph sets are not Euclidean spaces. A standard approach is to consider the Fr\'echet sample mean, and the Fr\'echet total sample variance.

The choice of metric is crucial to the computation of the Fr\'echet mean, since each metric induces a different mean graph. The Fr\'echet mean of graphs has been studied in the context where the distance is the edit distance (e.g.,
\cite{BNY20,G12,J16,JO09, JMB01} and references therein).  The edit distance reflects small scale  changes in the graphs and therefore the Fr\'echet mean will be sensitive to the fine structural variations between graphs. Effectively, the Fr\'echet mean with respect to the edit distance can be interpreted as an average of the fine structures in the observed graphs.

In this paper, we consider that the fine scale, which is defined by the local connectivity at the level of each vertex, may be intrinsically random and the quantification of such random fluctuations is uninformative when comparing graphs.  We therefore study a distance that can detect large scale patterns of connectivity that happen at multiple scales (e.g., community structure \cite{A17, LGT14}, modularity \cite{GN02}).

The adjacency spectral distance, which we define as the $\ell_2$ norm of the difference between the spectra of the adjacency matrices of the two graphs of interest \cite{WZ08}, exhibits good performance when comparing various types of graphs \cite{WM20}, making it a reliable choice for a wide range of problems. Spectral distances also exhibit practical advantages, as they can inherently compare graphs of different sizes and can compare graphs without known vertex correspondence (see e.g., \cite{FSS05, FVSRB10} and references therein). 

The adjacency spectrum is well-understood, and is perhaps the most frequently studied graph spectrum
(e.g., \cite{F01,FFF03}). The eigenvalues of the adjacency matrix carry important topological information
about the graph at different structural scales \cite{M85,WZ08}. The spectrum reveals information about large scale features such as community structure \cite{LGT14} or the existence of highly connected ``hubs'' \cite{FFF03}, as well as the smaller scale structure such as local connectivity (i.e. the degree of a vertex) or the ubiquity of substructures such as triangles \cite{F01}. In \cite{WM20}, the authors observe that the adjacency spectral pseudo-distance exhibits good performance across a variety of scenarios, making it a reliable choice for a wide range of problems (from the two-sample test problem to the change point detection problem).

A stronger notion of spectral similarity involves the similarity between the eigenspaces of the adjacency matrices of
the respective graphs (e.g., \cite{deng20, jin20, loukas19}). These notions of spectral similarity have regained some
interest in the context of graph coarsening (or aggregation) for graph neural
networks \cite{bravo-hermsdorff19, chenj21, li18, on20}.

Inspired by the conjecture of Vu \cite{V14}, we propose to use the eigenvalues of the adjacency matrix as
``fingerprints'' that uniquely (up to trivial permutations and the possible iso-spectral graphs) characterize a graph
(see also \cite{JS12}).

In practice, it is often the case that only the first $c$ eigenvalues are compared, where $c\ll n$. We still refer to such distances as spectral distances but comparison using the largest $c$ eigenvalues for small $c$ allows one to focus on the global structure of the graph while ignoring the local structure \cite{LGT14}. We provide in Section \ref{sec:Approx} our recommendation for the choice of $c$.

Instead of solving the minimization problem associated with the computation of the Fr\'echet mean in the original set $\cG$, the authors in \cite{FVSRB10} suggest to embed the graphs in Euclidean space, wherein they can trivially find the mean of the set. Because the embedding in \cite{FVSRB10} is not an isometry, there is no guarantee that the inverse of the average embedded graphs be equal to the Fr\'echet mean. Furthermore, the inverse embedding may not be available in closed form.  In the case of simple graphs, the Laplacian matrix of the graph uniquely characterizes the graph. The authors in \cite{GLBRK17} define the mean of a set of graphs using the Fr\'echet sample mean (computed on the manifold defined by the cone of symmetric positive semi-definite matrices) of the respective Laplacian matrices.
\section{Main Contributions.}
The Fr\'echet mean graph has become a standard tool for the analysis of graph-valued data. In this paper, we derive a method to compute the sample Fr\'echet mean when the distance between graphs is computed by comparing the spectra of the adjacency matrices of the respective graphs. We provide a rigorous theoretical analysis of our algorithm that demonstrates that our estimator converges toward the true sample Fr\'echet mean in the limit of large graph sizes. This is the first computation of the sample Fr\'echet mean for graphs when considering a spectral distance. This novel theoretical result relies on a combination of two ideas: stochastic block models provide universal approximants in the spectral adjacency pseudometric, and the dominant eigenvalues of the adjacency matrix of a stochastic block model can be used to recover the corresponding graph. We use numerical simulations to compare our theoretical analysis to the finite graph size estimates obtained with our algorithm.
\section{Notations.}
We denote by $G = (V,E)$ a graph with vertex set $V = \lbrace 1,2,...,n \rbrace$ and edge set $E \subset V \times V$. For
vertices $i,j \in V$ an edge exists between them if the pair $(i,j) \in E$. The size of a graph is called
$n = \vert V \vert$ and the number of edges is $m = \vert E \vert$. The density of a graph is called
$\rho_n = \frac{m}{n(n-1)/2}$. 

The matrix $\bA$ is the adjacency matrix of the graph and is defined as
\begin{align}
  \bA_{ij} = 
  \begin{cases}
    1 \quad \text{if }(i,j) \in E,\\
    0 \quad \text{else.}
  \end{cases}
\end{align}
We define the function $\sigma$ to be the mapping from the set of $n \times n$ adjacency matrices (square, symmetric matrices with zero entries
on the diagonal), $\M_{n \times n}$ to $\R^n$ that assigns to an adjacency matrix the vector of its $n$ sorted eigenvalues,
\begin{align}
  \sigma:        \M_{n\times n} & \longrightarrow \R^n,\\
  \bA & \longmapsto \blamb = [\lambda_1,\ldots,\lambda_n],
\end{align}
where $\lambda_1 \geq \ldots \geq \lambda_n$. Because we often consider the $c$ largest eigenvalue of the
adjacency matrix $\bA$, we define the mapping to the truncated spectrum as $\sigma_c$ ,
\begin{align}
  \sigma_c:        \M_{n\times n} & \longrightarrow \R^c,\\
  \bA & \longmapsto \blamb_c = [\lambda_1,\ldots,\lambda_c].
\end{align}
\begin{Definition}
  We define the adjacency spectral pseudometric as the $\ell_2$ norm between
  the spectra of the respective adjacency matrices,
  \begin{align} 
    d_A(G,G') = ||\sigma(\bA)-\sigma(\bA')||_2. \label{distance-adj}
  \end{align}
  
\end{Definition}
The pseudometric $d_A$ satisfies the symmetry and triangle inequality axioms, but not the identity axiom. Instead, $d_A$
satisfies the reflexivity axiom
\begin{equation*}
  d_A(G,G) = 0, \quad \forall G \in \cG.
\end{equation*}

When the adjacency matrices (or Laplacian) of graphs have a similar spectra it can be shown that the graphs have similar global and structural properties \cite{WM20}. As a natural extension of this spectral metric, sometimes only the largest $c$ eigenvalues are measured where $c \ll n$. We refer to this next metric as a truncation of the adjacency spectral pseudometric.

\begin{Definition}
  We define the truncated adjacency spectral pseudometric as the $\ell_2$ norm between
  the largest $c$ eigenvalues of the respective adjacency matrices, 
  \begin{align}
    d_{A_c}(G,G') = ||\sigma_c(\bA)-\sigma_c(\bA')||_2. \label{distance-trunc}
  \end{align}
\end{Definition}
\begin{Definition}
  \label{def:SetOfGraphs}
  We denote by  $\cG$  the set of all simple unweighted graphs on $n$ nodes.
\end{Definition}

\subsection{Random Graphs} 
We denote by $\cM (\cG)$ the space  of probability measures on $\cG$. In this work, when we talk about a measure we  always
mean a probability measure.
\begin{Definition}
  We define the set of random graphs distributed according  to $\mu$ to be the probability space $\left(\cG, \mu\right)$.
\end{Definition}
\begin{Remark}
  In this paper, the $\sigma$-field associated with the $\left(\cG, \mu\right)$ will always be the power set of $\cG$.
\end{Remark}
This definition allows us to unify various ensembles of random graphs (e.g., \ER, inhomogeneous \ER, Small-World, Barabasi-Albert etc) through the unique concept of a probability space.
\subsubsection{Kernel Probability Measures}
\label{subsec:kernProbMeas}
\noindent Here we define an important class of probability measures for our study.
\begin{Definition}
  \label{def:KernProbMeas}
  A probability measure $\mu \in \cM (\cG)$ is called a kernel probability measure if there exist a positive constant $\rho_n$ and a function $f$,
  \begin{equation}
    \rho_n f: [0,1]\times[0,1] \mapsto [0,1],
  \end{equation}
  such that $f(x,y) = f(1-y,1-x)$, and 
  \begin{align*}
    \forall G \in \cG, \text{with adjacency matrix}\; \bA=\left(a_{ij}\right), \\
    \mu\left( \left\{ \bA \right \}\right) = \mspace{-12mu} \prod_{1 \leq i < j \leq n}\mspace{-12mu}\prob{a_{ij}} = \mspace{-12mu} \prod_{1 \leq i < j \leq
      n}\mspace{-12mu} \bern{\rho_n f(\frac i n,\frac j n)}.
  \end{align*}
  The function $\rho_n f$ is called a kernel of $\mu$ and the probability measure is denoted $\mu_{\rho_n f}$.
\end{Definition}
\begin{Remark}
  We refer to these measures as kernel probability measures since the kernels naturally give rise to linear integral operators with kernels $f$.
\end{Remark}
\noindent We note that given the sequence $\left \{\frac i n \right \}_{i=1}^n$ and the measure $\mu$, the kernel $\rho_n f$ forms an
equivalence class of functions, characterized by their values on the grid $\left \{\frac i n \right \}_{i=1}^n \times \left \{\frac j n \right \}_{j=1}^n$.

\begin{Definition}
  We denote by $G_{\mu}$ a random realization of a graph $G \in \left(\cG,\mu\right)$.
\end{Definition}

\subsubsection{Stochastic Block Models
  \label{subsec:sbm}}
\noindent The stochastic block model (see \cite{A17}) plays an important role in this work. We review the specific features of this model using the notations that were defined in the previous paragraphs. The key aspects of the model are: the geometry of the blocks, the within-community edges densities, and the across-community edge densities. An example of the kernel function and associated adjacency matrix from a stochastic block model is given in Fig.~\ref{fig:Ex-f}.

We denote by $c$ the number of communities in the stochastic block model. The {\noindent \bf geometry} of the stochastic block model is encoded using the relative sizes of the communities. We denote by $\bs \in \ell_1$ a non-increasing non-negative sequence of relative community sizes with $c$ non-zero entries and $||\bs||_1 = 1$.

For the geometry specified by $\bm s$ we define an associated {\bf edge density} vector $\bp \in \ell_\infty$ such that $0<p_i$ for $i = 1,...,c$ and $p_i = 0$ for $i > c$ which describes the within-community edge densities. 

Finally, we denote by $\bQ = (q_{ij})$ an infinite matrix of cross-community edge densities where $q_{i,i} = 0$, $q_{i,j}  = q_{j,i}$, and $q_{i,j} = 0$ if $i > c$ or $j > c$.

\begin{Remark}
  We allow for infinite vectors with finite number of non-zero entries so that we may allow for the smooth introduction of new communities within the stochastic block model. For example, let $t \in [0,1]$ and parametrize $\bs$ and $\bp$ by $t$ as
  \begin{equation}
    \bs(t) = \begin{bmatrix}1 - t/2 \\ t/2 \\ 0 \\ \vdots\end{bmatrix} \quad \rho_n \bp(t) = \begin{bmatrix} 0.2 + t/2 \\ 0.1+t/2 \\ 0 \\ \vdots \end{bmatrix}.
  \end{equation}
\end{Remark}

We can parameterize a stochastic block model using one representative of the equivalence class of kernels, $f$, which we call the canonical stochastic block model kernel.

\begin{Definition}[Canonical stochastic block model kernel]
  The function $f$, which is piecewise constant over the blocks, and is defined by\\ $\rho_n f: [0,1] \times [0,1]  \longrightarrow [0,1]$
  \begin{align}
    (x,y) \longmapsto
    \begin{cases}
      \rho_n p_i & \text{if} \quad \sum_{k=1}^{i-1}s_k  \leq x <  \sum_{k=1}^{i} s_k, \; \\ & \text{and} \quad 
      \sum_{k=1}^{i-1} s_k  \leq y <  \sum_{k=1}^{i} s_k,\\
      \rho_n q_{ij} & \text{if} \quad \sum_{k=1}^{i-1}s_k  \leq x <  \sum_{k=1}^{i} s_k, \; \\ & \text{and} \quad 
      \sum_{k=1}^{j-1} s_k \leq y <  \sum_{k=1}^{j} s_k
    \end{cases}
  \end{align}
  is called the canonical kernel of the stochastic block model with measure
  $\mu$ (see, e.g. Fig.~\ref{fig:Ex-f}), and we denote it by $f(x,y,\bp,\bQ,\bs)$. \\
\end{Definition}
\begin{Definition}[Set of canonical stochastic block model kernels]
  We denote the set of all canonical stochastic block model kernels $f$ as
  \begin{equation}
    \mathcal{F} = \lbrace f \vert f \text{ is a canonical stochastic block model kernel.} \rbrace
  \end{equation}
\end{Definition}	
\noindent \textit{Example 1.} Given $\bm s = \begin{bmatrix} 1/2 & 1/4 & 1/4 & 0  \cdots \end{bmatrix}^T$ the values of $f(x,y; \bm p, \bQ, \bm s)$ in the unit square are shown in Fig. \ref{fig:Ex-f}.

\begin{figure}
  \centering
  \begin{minipage}{.5\textwidth}
    \centering
    \includegraphics[clip, trim = 9.95cm 3.25cm 9.5cm 2.25cm, width = 4.25cm, height = 4.2cm]{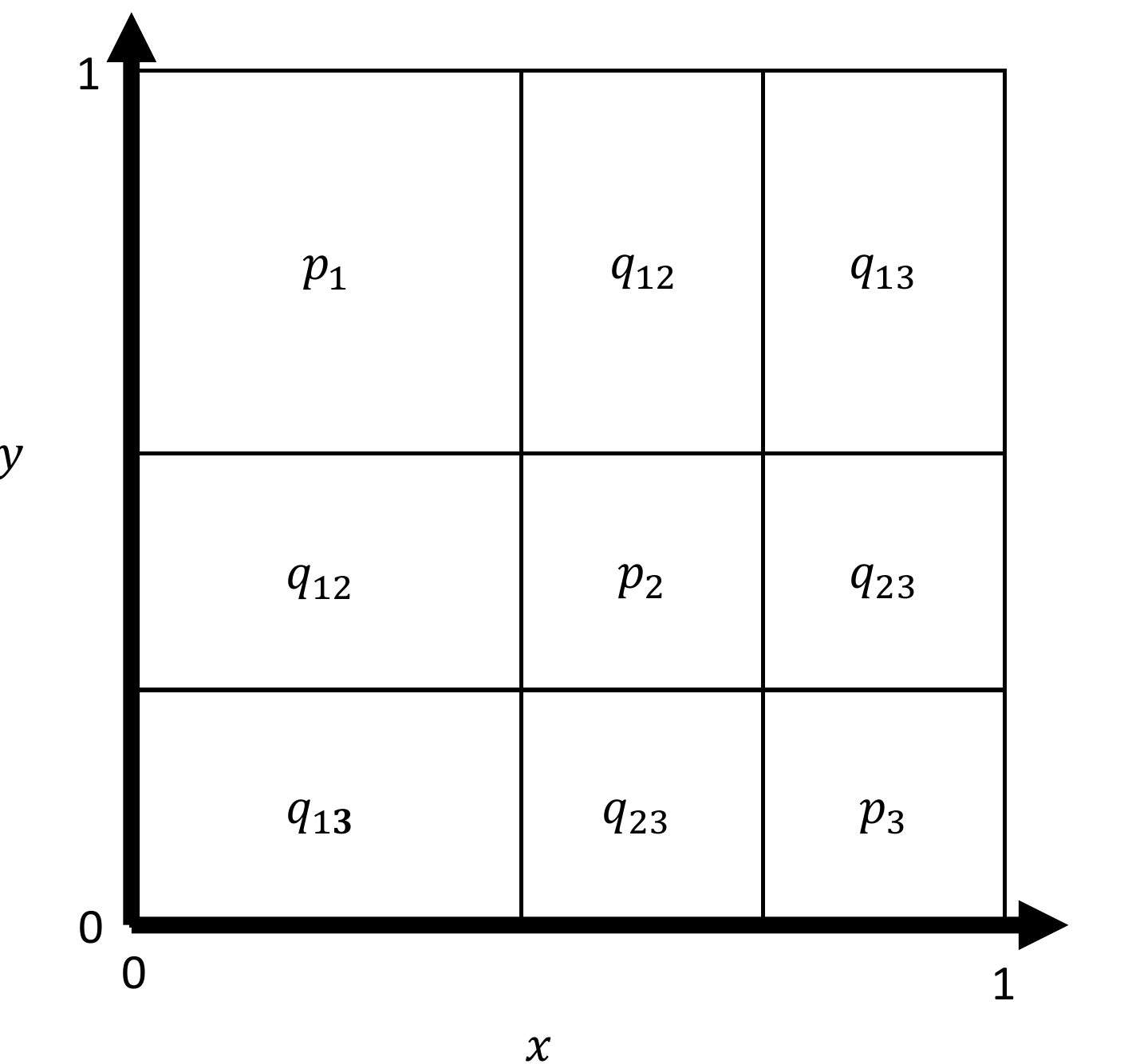}
    \caption{\textit{Example stochastic block model kernel $f(x,y; \bp, \bm Q, \bs)$}}
    \label{fig:Ex-f}
  \end{minipage}%
  \begin{minipage}{.5\textwidth}
    \centering
    \includegraphics[clip, trim = 5cm 8cm 4.75cm 8.75cm, width= 3.5cm, height = 4.15cm]{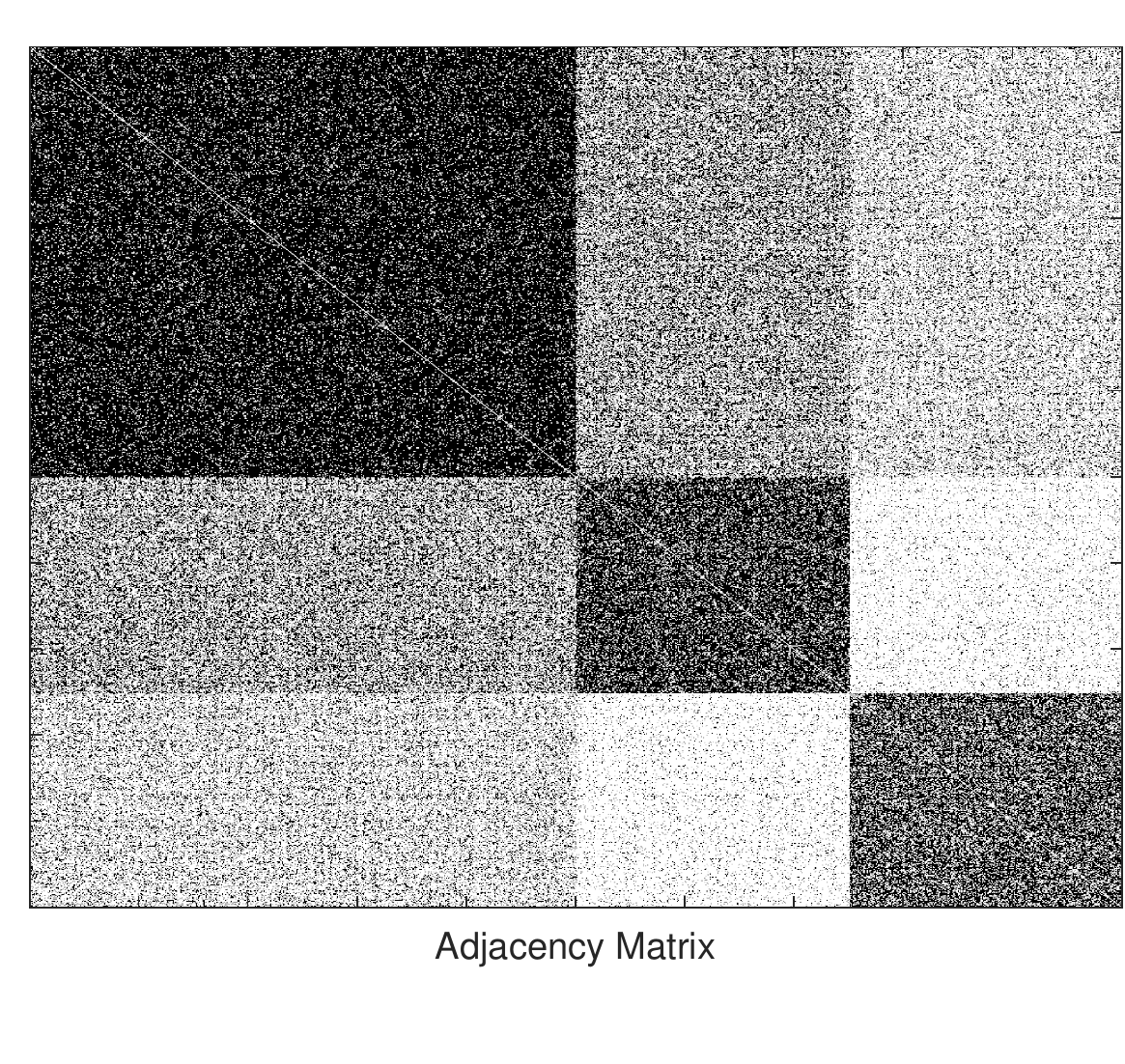}
    \caption{\textit{Example adjacency matrix}}
  \end{minipage}
\end{figure}

\section{The Fr\'echet mean and sample Fr\'echet mean}
\label{sec:FM&EFM}
This section specifies the minimization procedures associated with the Fr\'echet mean and sample Fr\'echet mean problem. We equip the set $\cG$ of graphs defined on $n$ vertices (see definition \ref{def:SetOfGraphs}) with the pseudometric defined by the
$\ell_2$ norm between the spectra of the respective adjacency matrices, $d_{A_c}$, (see (\ref{distance-trunc})). We consider a
probability measure $\mu \in \cM(\cG)$ that describes the probability of obtaining a given graph when we sample $\cG$ according to
$\mu$. Using $d_{A_c}$, we quantify the spread of the graphs, and we define a notion of centrality, which gives the
location of the expected graph, according to $\mu$.

\begin{Definition}[Fr\'echet mean \cite{F48}]
  The Fr\'echet mean of the probability measure $\mu$ in the pseudometric space $(\cG,d_{A_c})$ is the set of graphs $G^*$ whose expected distance squared to $\cG$ is minimum,
  \begin{equation}
    \left\{G^* \in \cG\right\} =  \argmin{G \in \cG} \E{d_{A_c}^2(G,G_{\mu})}[\mu], \label{def:FM-general}
  \end{equation}
\end{Definition}
where $G_\mu$ is a random realization of a graph distributed according to the probability measure $\mu$, and the expectation
$\E{d^2(G,G_{\mu})}[\mu]$ is computed with respect to the probability measure $\mu$. \\

Because $\cG$ is a finite set, the minimization problem (\ref{def:FM-general}) always has at least one solution. Throughout this work, we are interested in determining at least one element of the set $\lbrace G^* \in \cG \rbrace$. Since our results hold for any minimizer of (\ref{def:FM-general}) (i.e. for any Fr\'echet mean of $\mu$), to ease the exposition, and without loss of generality, we assume that the Fr\'echet mean is unique. We therefore write $\left\{G^* \in \cG\right\}$ as a singleton and we denote the Fr\'echet mean as
\begin{align}
  G^* = \underset{G \in \cG}{\text{argmin }}\E{d_{A_c}^2(G,G_{\mu})}[\mu]. \label{def:FM} 
\end{align}

We note the similarity between equation (\ref{def:FM}) and the definition of the barycenter \cite{P06}. Indeed, as we change $\mu$ we expect that, for a fixed $G$, $\E{d_{A_c}^2(G,G_{\mu})}[\mu]$ will change, and therefore the Fr\'echet mean, $G^*$, will move inside $\cG$ for different choices of the probability measure $\mu$. Observe that $G^*$ plays the role of the center of mass, for the mass distribution associated with $\mu$. 

In practice, the only information known about a distribution on $\cG$ comes from a sample of graphs. Therefore, we need a notion
of Fr\'echet sample mean, which is defined by replacing $\mu$ with the empirical measure.
\begin{Definition}[Sample Fr\'echet mean] 
  Let $\left\{\Gk\right\}\; 1 \leq k \leq N$ be a set of graphs in $\cG$. The sample Fr\'echet mean is defined by
  \begin{align}
    G_N^* = \argmin{G \in \cG} \frac 1 N \sum_{k=1}^N d^2(G,\Gk). \label{eqn:EFM}
  \end{align}
\end{Definition}
Again, our results hold for any minimizer of (\ref{eqn:EFM}) and so without loss of generality we assume that the sample Fr\'echet mean is unique for any given set of graphs. The dependence on $N$ is explicitly written here but may be suppressed throughout the paper when it is obvious. 

The computation of $G_N^*$ for sets of large graphs is intractable due to two primary issues. The first is that $\left|\cG\right| = \mathcal{O}(2^{n^2})$ so any brute force procedure to solve the minimization problem in (\ref{eqn:EFM}) will not compute in reasonable time. Second, the set $\cG$ is not ordered so searching the space of graphs in a principled manner is difficult (in contrast to the situation with trees \cite{B14, BHV01}). We suggest solving these issues by first lifting the sample Fr\'echet mean problem to $\cM(\cG)$ and defining an approximation to the lifted problem. Our specific approach involves searching for the correct parameters of a canonical stochastic block model kernel, $f$, such that the sample Fr\'echet mean given $\mu_{\rho_n f}$ almost surely approximates the target graph, $G_N^*$, with respect to $d_{A_c}$.

\section{Approximation of the sample Fr\'echet mean.}
\label{sec:Approx}
We first state the primary theoretical results, Theorem \ref{thm:SpecSimLargeGraphs} and Corollary \ref{cor:ApproxEmpFM}, which constitute the main contributions of this work and which form the foundation of our algorithm (see Alg. \ref{alg:DetFM_C}). We additionally state in this section theorems that are necessary results for the implementation of our algorithm.

Let $G \in \cG$ with adjacency matrix $\bA$ such that $n^{-2/3} \ll \rho_n \ll 1$. Assume that 
\begin{equation}
  0 \preccurlyeq \sigma_c(\bA)
\end{equation}
and for every $1 \leq i \neq j \leq c$, $\lambda_i \neq \lambda_j$. Our primary theoretical contribution, Theorem \ref{thm:SpecSimLargeGraphs}, states that we may approximate any graph $G$ that satisfies our assumptions by the sample Fr\'echet mean of an appropriate stochastic block model kernel probability measure, $\mu_{\rho_n f}$, almost surely with respect to the truncated adjacency spectral pseudo-metric.
\begin{theorem}[Spectrally similar large graphs]
  \label{thm:SpecSimLargeGraphs}
  $\forall \epsilon > 0$, $\exists n_1 \in \N$ such that $\forall n > n_1$, $\exists f(x,y; \bp, \bQ, \bs)$ a canonical stochastic block model kernel with $c$ communities such that
  \begin{equation}
    \lim_{N \to \infty} d_{A_c}(G,G_{N, \mu_{\rho_n f}}^*) < \epsilon \quad a.s.
  \end{equation}
  where $G_{N, \mu_{\rho_n f}}^*$ denotes the sample Fr\'echet mean of $\lbrace \Gk \rbrace_{k=1}^N$, an iid sample distributed according to $\mu_{\rho_n f}$.
\end{theorem}
\begin{proof}
  The proof is in \ref{app:SpecSimLargeGraphs}.
\end{proof}
\begin{Remark}
  \label{rem:Geom}
  While we are free to choose the geometry vector $\bs$, we make the choice that $s_1 \geq s_i$ for $i = 2,...,c$ and $s_i = s_j$ for $i,j = 2,...,c$. This choice is not necessary and any choice of the vector $\bs$ would be suitable (see e.g. subsection \ref{subsec:VarComSize}).
\end{Remark}
The following corollary applies Theorem \ref{thm:SpecSimLargeGraphs} to the sample Fr\'echet mean of any given data set of graphs, $\lbrace \Gk \rbrace_{k=1}^N$. It states that for any given set of graphs whose sample Fr\'echet mean, $G_N^*$, satisfies the assumptions of Theorem \ref{thm:SpecSimLargeGraphs}, there exists a canonical stochastic block model kernel defining a probability measure, $\mu_{\rho_n f}$, where the sample Fr\'echet mean of an iid sample from $\mu_{\rho_n f}$, denoted $G_{\tilde N, \mu_{\rho_n f}}^*$, is almost surely close to $G_N^*$. This corollary forms the basis of our approach to solving the sample Fr\'echet mean problem, equation (\ref{eqn:EFM}).

Let $\lbrace \Gk \rbrace_{k=1}^N$ be a set of graphs with sample Fr\'echet mean $G_N^*$. Assume $G_N^*$ satisfies the assumptions of Theorem \ref{thm:SpecSimLargeGraphs}.
\begin{corollary}[Approximation of the sample Fr\'echet mean]
  \label{cor:ApproxEmpFM}
  $\forall \epsilon > 0$, $\exists n_1 \in \N$ such that $\forall n > n_1$, $\exists f(x,y; \bp, \bQ, \bs)$ a canonical stochastic block model kernel with $c$ communities such that
  \begin{equation}
    \lim_{\tilde N \to \infty} d_{A_c}(G_N^*,G_{\tilde{N},\mu_{\rho_n f}}^*) < \epsilon \quad a.s.
  \end{equation}
  where $G_{\tilde N, \mu_{\rho_n f}}^*$ denotes the sample Fr\'echet mean of $\lbrace \tilde G^{(k)} \rbrace_{k=1}^{\tilde N}$, an iid sample distributed according to $\mu_{\rho_n f}$.
\end{corollary}

\begin{Remark}
  \label{rem:Density}
  One  requirement on $G_N^*$ is that the density, denoted $\rho_n^*$, satisfies $n^{-2/3} \ll \rho_n^* \ll 1$. The theory in \cite{FM21} states that as long each graph in our sample set $\lbrace \Gk \rbrace_{k=1}^N$ satisfies this density condition, then so to does $G_N^*$.
\end{Remark}

Theorem \ref{thm:SpecSimLargeGraphs} and Corollary \ref{cor:ApproxEmpFM} suggest that the sample Fr\'echet mean can be approximated by a large sample Fr\'echet mean from a stochastic block model in the limit of large graph size. This result is purely existential and does not provide an algorithm to construct the sequence of stochastic block model kernels.

We now address how to determine the correct canonical stochastic block model kernel $f$ with the following three theorems. The strategy to determine the kernel $f$ begins with the fact that the eigenvalues, $\sigma_c(\bA_N^*)$, concentrate around the sample mean spectrum, $\frac 1 N \sum_{k=1}^N \sigma_c(\bA^{(k)})$ (Theorem \ref{thm:GeomEigsFM}). We can therefore find the canonical stochastic block model kernel $f$ where the sample Fr\'echet mean from $\mu_{\rho_n f}$ has an adjacency matrix whose eigenvalues best approximate the sample mean spectrum, $\frac 1 N \sum_{k=1}^N \sigma_c(\bA^{(k)})$, by utilizing the estimates given in Theorems \ref{thm:AlmostSure} and \ref{thm:EstLargeEigs}. Once we have estimated the parameters of the canonical stochastic block model kernel $f$ we estimate the sample Fr\'echet mean given $f$, $G_{\tilde N, \mu_{\rho_n f}}^*$, by sampling graphs distributed according to $\mu_{\rho_n f}$ and computing the set mean graph (Theorem \ref{thm:HighProbStat}).

Let $\lbrace \Gk \rbrace_{k=1}^N$ be a set of graphs with sample Fr\'echet mean $G_N^*$ with adjacency matrix $\bA_N^*$. Our following theorem shows that the eigenvalues of the adjacency matrix $\bA_N^*$ concentrates around the sample mean spectrum.
\begin{theorem}
  \label{thm:GeomEigsFM}
  $\forall \epsilon > 0$, $\exists n^* \in \N$ such that $\forall n > n^*$,
  \begin{equation}
    ||\sigma_c(\bA_N^*) - \frac 1 N \sum_{k=1}^N \sigma_c(\bA^{(k)}) ||_2 < \epsilon.
  \end{equation}
\end{theorem}
\begin{proof}
  The proof is in \ref{app:GeomEigsFM}
\end{proof}

Let $\lbrace \tilde G \rbrace_{k=1}^{\tilde N}$ be an iid sample of graphs distributed according to $\mu_{\rho_n f}$ with sample Fr\'echet mean $G_{\tilde N, \rho_n f}^*$. Let $\bA_{\tilde N, \rho_n f}^*$ be the adjacency matrix of $G_{\tilde N, \rho_n f}^*$. Our next two theorems show how we estimate the eigenvalues, $\sigma_c(\bA_{\tilde N, \rho_n f}^*)$, in terms of the kernel function $f$. We first show that the expected eigenvalues, $\E{\sigma_c(\bA_{\mu_{\rho_n f}})}$, are almost surely within $\epsilon$ of $\sigma_c(\bA_{\tilde N, \rho_n f}^*)$. 

\begin{theorem}[The Eigenvalues of the sample Fr\'echet Mean of Stochastic Block Models]
  \label{thm:AlmostSure}
  $\forall \epsilon > 0$, $\exists n^* \in \N$ such that for all $n > n^*$,
  \begin{equation}
    \lim_{\tilde N \to \infty} || \sigma_c(\bA_{\tilde N, \rho_n f}^*) - \E{\sigma_c(\bA_{\mu_{\rho_n f}})}||_2 < \epsilon \quad a.s.
  \end{equation}
\end{theorem}
\begin{proof}
  The proof is in \ref{app:SpecSimLargeGraphs}.
\end{proof}

Since we do not have a closed form expression for $\E{\sigma_c(\bA_{\mu_{\rho_n f}})}$, Theorem \ref{thm:EstLargeEigs} shows that we can estimate the expected eigenvalues, $\E{\sigma_c(\bA_{\mu_{\rho_n f}})}$, in terms of the kernel function $f$. It should also be noted that the following theorem is a modification of Theorem 2.4 in \cite{CCH20}.

Let $\mu_{\rho_n f} \in \cM(\cG)$ be a kernel probability measure with kernel $f$. Let $L_f$ be the linear integral operator with the same kernel function, $f$. Assume $L_f$ has a finite rank of $c$. Denote the eigenvalues and eigenfunctions of $L_f$ as $\lambda_i(L_f)$ and $r_i(x)$ respectively where for each $i = 1,...,c$, $r_i(x)$ is assumed to be piecewise Lipschitz with finitely many discontinuities.

\begin{theorem}[Estimation of the Largest Eigenvalues of Stochastic Block Models] \hfill \\
  \label{thm:EstLargeEigs}
  For $i = 1,...,c$
  \begin{equation}
    \E{\lambda_i(\bA_{\mu_{\rho_n f}})} = \lambda_i(L_f) n \rho_n + \mathcal{O}(\sqrt{\rho_n})
  \end{equation}
\end{theorem}
\begin{proof}
  The proof is in \ref{app:SpecSimLargeGraphs}.
\end{proof}

Recall that our goal is to approximately compute $G_N^*$ by finding a stochastic block model kernel $f$ and determining the sample mean of graphs distributed according to $\mu_{\rho_n f}$. To determine the canonical stochastic block model kernel, Theorems \ref{thm:GeomEigsFM} - \ref{thm:EstLargeEigs} show that we may align $n \rho_n \lambda_i(L_f)$ with the sample mean spectrum $\frac 1 N \sum_{k=1}^N \sigma_c(\bA^{(k)}$. We therefore search for an $f$ that solves the following minimization problem,
\begin{align}
  f^* = \underset{f \in \mathcal{F}}{\text{argmin }} \sum_{i=1}^c \left| n \rho_n \lambda_i(L_f) - \frac 1 N \sum_{k=1}^N \lambda_i(\bA^{(k)})\right|^2. \label{eqn:DetF}
\end{align}

Given the canonical stochastic block model kernel $f$ that solves equation (\ref{eqn:DetF}), Theorem \ref{thm:HighProbStat} shows a method of estimating the sample Fr\'echet mean of graphs distributed iid according to $\mu_{\rho_n f}$ by sampling from $\mu_{\rho_n f}$.

Let $\lbrace \tilde G^{(k)} \rbrace_{k = 1}^{\tilde N}$ be a sample of graphs distributed according to $\mu_{\rho_n f}$ where $f$ is the canonical stochastic block model kernel. Define the set mean graph by
\begin{equation}
  \widehat{G}_{\tilde N,\mu_{\rho_n f}}^* = \argmin{\tilde G \in \lbrace \tilde{G}^{(k)} \rbrace_{k = 1}^{\tilde N}} \frac{1}{\tilde N} \sum_{k = 1}^{\tilde N} d_{A_c}^2(\tilde G,\tilde{G}^{(k)}) 
\end{equation}
with adjacency matrix $\hat{\bA}_{\tilde N,\mu_{\rho_n f}}^*$.
\begin{theorem}[Convergence in probability of the truncated spectrum of the set mean graph]\hfill \\
  \label{thm:HighProbStat}
  $\forall \epsilon > 0$, 
  \begin{equation}
    \lim_{n \to \infty} P(||\sigma_c(\hat{\bA}_{\tilde N, \mu_{\rho_n f}}^*) - \E{\sigma_c(\bA_{\mu_{\rho_n f}})}||_2 > \epsilon) = 0. \label{eqn:HighProbClose}
  \end{equation}
\end{theorem}
\begin{proof}
  The proof is in \ref{app:HighProbStat}.
\end{proof}

Due to the convergence in distribution to a multivariate normal of the eigenvalues of adjacency matrices from the stochastic block model, (see Theorem 2.3 in \cite{CCH20} and Corollary 1 in \cite{ACT21}) we observe that a relatively small size of $\tilde N$ is needed in the finite graph case. In our experiments in section \ref{sec:Exp} we take $\tilde N = 5$.

The practical significance of our theoretical analysis is the invention of an algorithm to approximate the solution to the sample Fr\'echet mean problem, equation (\ref{eqn:EFM}), which we give the pseudo-code for below. 

\subsection{Summary and Algorithm}
Given a finite sample of graphs $\lbrace \Gk \rbrace_{k=1}^N$, our theory allows us to estimate the sample Fr\'echet mean graph, $G_N^*$, by solving an approximate problem in two steps:
\begin{enumerate}
\item Identify the correct canonical stochastic block model kernel $f$ by solving equation (\ref{eqn:DetF}).
\item Estimate $G_{\tilde N, \mu_{\rho_n f}}^*$ using Theorem \ref{thm:HighProbStat} taking $\tilde N$ as large as desired.
\end{enumerate}
A notable first step is to estimate $c$, the number of eigenvalues of $G_N^*$ to consider. We suggest estimating $c$ as per Alg. \ref{alg:DetC} and use this estimate in Alg. \ref{alg:DetFM_C}.

\begin{algorithm}[H]
  \caption{Determine $c$ for the approximate sample Fr\'echet mean}
  \label{alg:DetC}
  \begin{algorithmic}[1]
    \Require Set of graphs, $M = \lbrace \Gk \rbrace_{k=1}^N$, and integer $K$
    \State Compute the geometric average spectrum of graphs in $M$ as $\bar{\bm \lambda} = \frac 1 N \sum_{k=1}^N \blamb^{(k)}$.
    \State Initialize $i = 0$.
    \State \textbf{Do}
    \State \hspace{0.5cm} $i=i+1$
    \State \hspace{0.5cm} Initialize $r = \bar{\bm \lambda}(i)$
    \State \hspace{0.5cm} Initialize the semi-circle probability density function (see e.g. \cite{ACK15}), as $s(\lambda;r)$ where $r$ is the radius.
    \State \hspace{0.5cm} Assume $\bar{\bm \lambda}(j) \sim s(\lambda;r)$ for $j = i,...,n$.
    \State \hspace{0.5cm} Determine the PDF of the $K$ largest order statistics with a sample size $n-i$, $\lambda_{(n-i)},...,\lambda_{(n-i-K+1)}$
    \State \hspace{0.5cm} Compute the expected value of the $K$ largest order statistics from the PDF $s(\lambda;r)$ with a sample size of $n-i$.
    \State \hspace{0.5cm} With sample size $n-i$, compute the standard deviation of the $K$ largest order statistics, $\sigma_{n-i},...,\sigma_{n-i-K+1}$\\
    \textbf{While} $\vert \bm \bar{\lambda}(1+i) - \E{\lambda_{(n-i)}}\vert >\sigma_{n-i} \lor ... \lor \vert \bm \bar{\lambda}(K+i) - \E{\lambda_{(n-i-K)}}\vert >\sigma_{n-i-K+1}$
    \State \textbf{Return:} $c = i-1$
  \end{algorithmic}
\end{algorithm}
Alg. \ref{alg:DetC} assumes that all but the $c$ largest eigenvalues in the vector $\bar{\bm \lambda}$ follow a bulk distribution given by the semi-circle law (see e.g. \cite{ABH16, ACK15, EYY12} and references therein). We determine the edge of the bulk iteratively by assuming the edge is defined by the largest observed eigenvalue and compute whether the next $K$ sequential eigenvalues are within a standard deviation of their expected value. Upon termination, the number of eigenvalues left outside the bulk determines our choice for $c$. Note that any estimate of $c$ will suffice and the algorithm above is a suggestion. We make use of this estimate for $c$ in the following algorithm.

\begin{algorithm}[H]
  \caption{Approximate sample Fr\'echet mean}
  \label{alg:DetFM_C}
  \begin{algorithmic}[1]
    \Require Set of graphs, $M = \lbrace \Gk \rbrace_{k = 1}^N$
    \State Compute the average density $\bar{\rho}_n$ of the graphs in $M$
    \State Approximate $c$ via Alg. \ref{alg:DetC} and determine $\bs$ (see Remark \ref{rem:Geom}).
    \State For each $i = 1,...,c$ compute $\bar{\lambda}_i= \frac 1 N \sum_{k=1}^N \lambda_i(\bA^{(k)}).$
    \State Randomly initialize $\bm p$
    \State Initialize $\bQ = (q_{ij})$ such that $q_{ij} = q$ for all $i,j$ and enforce $||f(x,y; \bp, \bQ, \bs)||_1 = 1$
    \While{Relative change in $\bm p$ and $q$ is large}
    \State {Estimate the gradient of $\sum_{i=1}^c \left|n \bar{\rho}_n \lambda_i(L_f) - \bar{\lambda}_i\right|^2$ via centered differences.}
    \State{Update $\bm p$ via a projected gradient descent step}
    \State {Update $q$ such that $||f(x,y; \bp, \bQ, \bs)||_1 = 1$}
    \EndWhile
    \State {Estimate $G_{\tilde N,\mu_{\rho_n f}}^*$ as $\widehat{G}_{\tilde N,\mu_{\rho_n f}}^*$ (see Theorem \ref{thm:HighProbStat})}
    \State{\textbf{Return: } $\widehat{G}_{\tilde N,\mu_{\rho_n f}}^*$.}
  \end{algorithmic}
\end{algorithm}
The founding idea for Alg. \ref{alg:DetFM_C} is the following: Any graph $G$ can be expressed as the Fr\'echet mean of some probability measure. For large graphs, our theory shows we may search for a kernel probability measure, $\mu_{\rho_n f}$, by aligning the eigenvalues of $L_f$ (Steps 6 - 10), and then estimating the Fr\'echet mean of $\mu_{\rho_n f}$ using the set mean graph (Step 11).
\begin{Remark}
  Corollary \ref{cor:ApproxEmpFM} shows the existence of a canonical stochastic block model kernel, $f$. In our algorithm we choose to seek for a kernel $f$ with $||f||_1 =1$ so that the expected density of graphs drawn from $\mu_{\bar{\rho}_n f}$ is $\bar{\rho}_n$.
\end{Remark}

\subsection{Computational complexity of Algorithm \ref{alg:DetFM_C}, the numerical estimation of the sample Fr\'echet mean}
Step 11, determining $\widehat{G}_{\tilde N, \mu_{\rho_n f}}^*$, is the most computationally expensive with a time complexity of $\cO(\tilde N n^2 c^3)$. This time complexity is because we generate $\tilde N$ graphs on $n$ vertices and compute the $c$ largest eigenvalues of each to determine the most central element, $\widehat{G}_{\tilde N, \SBM}^*$.

\begin{Remark}
  \label{rem:SmallN}
  It should be noted that for large enough $n$, taking $\tilde N = 1$ provides a sufficient estimate and the computational time for step 11 is reduced to $\cO(n^2)$.
\end{Remark}

\section{Experimental Validation.}
\label{sec:Exp}
\subsection{Assessment, Validation, and Comparison}
A brute force computation of the Fr\'echet mean or median graph based on the adjacency spectral pseudo-distance is
unrealistic (it requires about $\Omega\left(n^22^{n^2}\right)$ operations for graphs of size $n$), and we therefore do
not provide a ground truth in our experiments (see section \ref{sec:Exp}). One may consider comparing the Fr\'echet mean
computed here to a Fr\'echet mean computed with respect to the edit distance for which several optimization algorithms
have been proposed (e.g., \cite{bardaji10,boria19,ferrer09,JO09,JMB01}). While this comparison may be feasible, it is
uninformative as the Fr\'echet mean with respect to the edit distance need not have any resemblance to the Fr\'echet
mean with respect to $d_{A_c}$.

All the code and data is provided at \url{https://github.com/dafe0926/approx_Graph_Frechet_Mean}. To the best of our
knowledge, this study provides the first algorithm to compute the sample Fr\'echet mean for a dataset of graphs when
considering a spectral distance, as a consequence, we have no baselines to compare our results with.

\subsection{Choice of the Datasets}
Graph-valued databases have recently been created and made available publicly
\cite{riesen08,IAMdatabase,morris20,TUdatasets}. These databases are designed for the evaluation of machine learning
algorithms (e.g., classification, regression, etc.) and the mean (or median) for each class is not provided (even for
the edit distance). Consequently, we believe that computing the Fr\'echet mean of these graphs ensembles provide little
scientific value for the purpose of validating our method.

Instead we present results of experiments conducted on synthetic datasets that are generated using ensembles of random
graphs. Ensembles of random graphs capture prototypical features of existing real world networks. Because our
theoretical analysis and associated algorithms rely on the (fixed community size) stochastic block model graphs as the
``atoms'' that are used to approximate any Fr\'echet mean, we expect that our algorithm will perform well when computing
the Fr\'echet mean of graphs generated by stochastic block models. Our experimental investigation is therefore concerned
with the performance of our approach in scenarios where the families of graph ensembles exhibit structural features that
are very different from those of the stochastic block models with  fixed community size.

We illustrate the theoretical analysis of the previous sections with experimental results using various synthetic datasets of
graphs. Each data set consists of $N = 50$ graphs on $n = 600$ nodes. We consider three different iid data sets of graphs, $M_1,...,M_3$, drawn from distributions $\mu_1,...,\mu_3$ respectively. The distributions have the following high level descriptions. 
\begin{center}
  \begin{tabular}{ l l}
    $\mu_1$: &Barabasi-Albert\\
    $\mu_2$: &Small world\\
    $\mu_3$: &Variable community size stochastic block model
  \end{tabular}
\end{center}
Note that $\mu_1$ and $\mu_2$ induce graphs with vastly different topologies than those generated by stochastic block models and yet we are still able to provide good approximations of the sample Fr\'echet mean. Within each subsection we discuss the specific parameters for each distribution when applicable. For each dataset, we determine the parameters of the stochastic block model whose sample Fr\'echet mean is close to the sample Fr\'echet mean of each dataset, $M_i$, and compute $\widehat G_{\tilde N, \mu_{\rho_n f}}^*$. 

\subsection{Barabasi-Albert approximate sample Fr\'echet mean} 
The probability measure in this section is associated with a Barabasi-Albert ensemble. The initial graph is fully connected on $m_0 = 5$ nodes and $m = 5$ edges were added at each step. In Fig. \ref{fig:baFM} we reorder the nodes based on their degree for the Barabasi-Albert graph to get a better visual understanding of the similarities between an observed graph and the approximate sample Fr\'echet mean. Our estimate for the number of communities is $c =12$.

Fig. \ref{fig:baFM} is a visual depiction of a graph from $M_1$ compared to the approximate sample Fr\'echet mean graph $\widehat G_{\tilde N, \mu_{\rho_n f}}^*$ though we note that there need not be any visual similarity between a graph in $M_1$ and $\widehat{G}_{\tilde N, \mu_{\rho_n f}}^*$ since any observation from a distribution $\mu$ need not be similar to the mean of $\mu$.

\begin{figure}[H]
  \centerline{
    \includegraphics[clip, trim = 0cm 0cm 0cm 0cm, width=\textwidth, height = 0.42\textwidth]{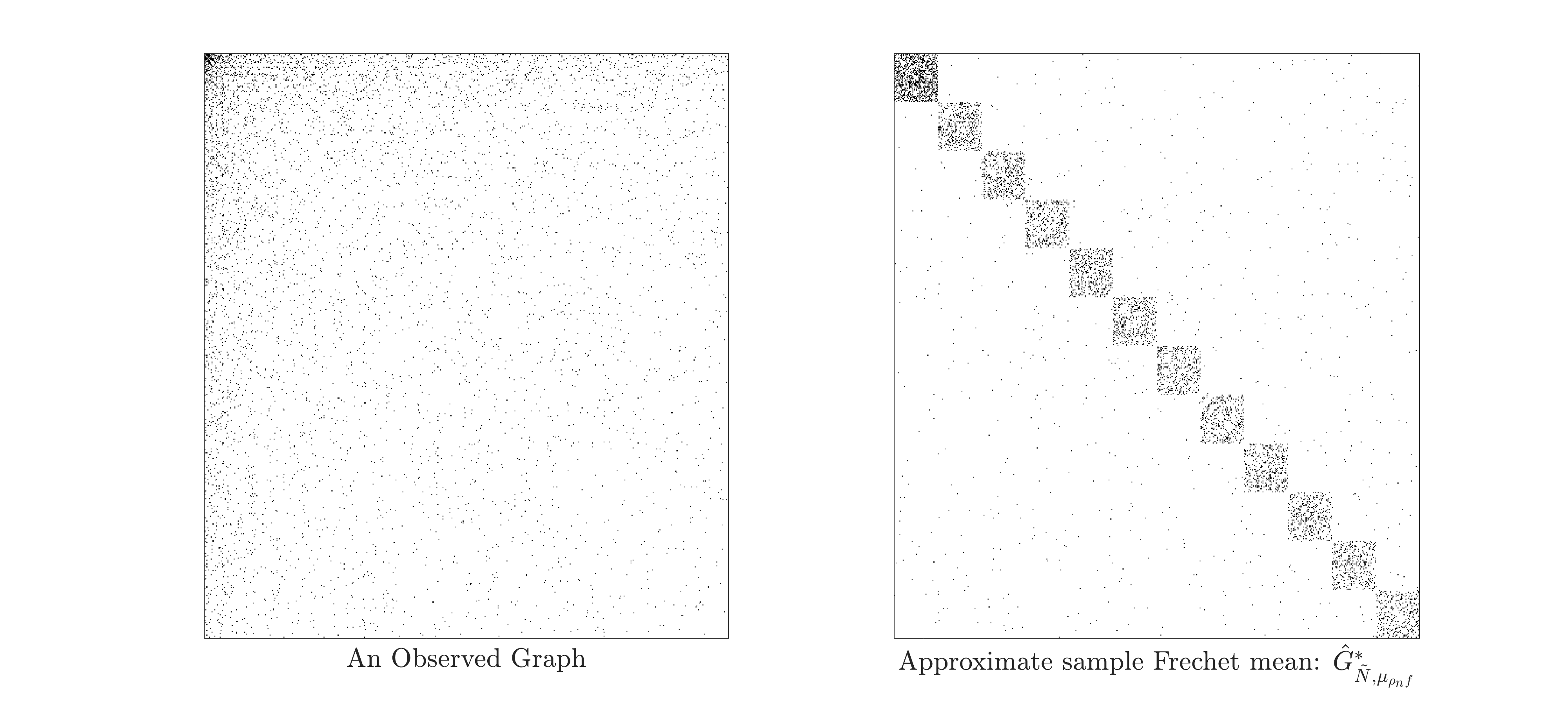}}
  \caption{\textit{Visualization of a graph in $M_1$ and the approximate sample Fr\'echet mean of $M_1$, $\hat G_{\tilde N, \mu_{\rho_n f}}^*$}}
  \label{fig:baFM}
\end{figure}
\begin{figure}[H]
  \centerline{
    \includegraphics[clip, trim = 0cm 0cm 0cm 0cm, width=\textwidth, height = 0.42\textwidth]{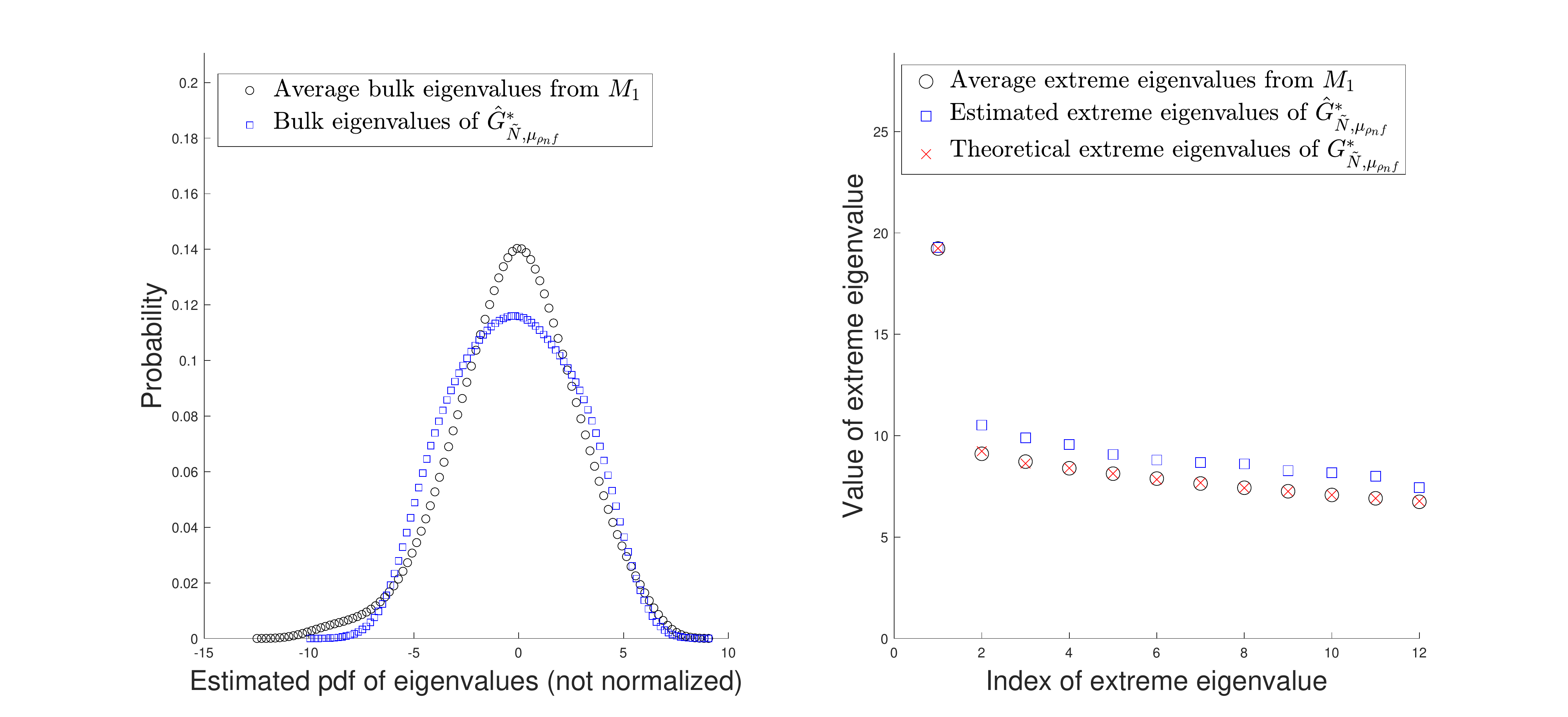}}
  \caption{\textit{\textbf{Left}: The average distribution of bulk eigenvalues from $M_1$ (black). The distribution of bulk eigenvalues of the approximate sample Fr\'echet mean $\widehat G_{\tilde N, \mu_{\rho_n f}}^*$ (blue). \textbf{Right}: The average extreme eigenvalues from $M_1$ (black). The expected extreme eigenvalues of $G_{\tilde N, \mu_{\rho_n f}}^*$ (red). The extreme eigenvalues of $\widehat G_{\tilde N, \mu_{\rho_n f}^*}$ (blue).}}
  \label{fig:baHist}
\end{figure}
Fig. \ref{fig:baHist} depicts the alignment of the spectra from the approximate sample Fr\'echet mean with that of the average
spectra of the graphs from set $M_1$. Note that the misalignment in the largest eigenvalues is due to the finite graph approximation.

The theory presented in Section \ref{sec:Approx} only ensures that the $c$ largest eigenvalues can be well approximated. In Fig. \ref{fig:baHist}, we see that the expected eigenvalues, (red markers), are near perfect estimates of the average extreme eigenvalues. However there is a notable distance between the extreme eigenvalues of $\widehat G_{\tilde N, \mu_{\rho_n f}}^*$, (blue markers), and the expected eigenvalues (red markers). This distance is determined primarily by the size of the graph $n$ and as $n$ increases this distance will decay like $\mathcal{O}(\sqrt{\rho_n})$ per Theorem \ref{thm:EstLargeEigs}.
\subsection{Small World approximate sample Fr\'echet mean}
The parameters for the Small World ensemble are the number of connected nearest neighbors, $K = 22$, and the probability of rewiring, $\beta = 0.7$.

Here we see a nice similarity between the adjacency matrices of the two graphs (see Fig. \ref{fig:swFM}). Furthermore Fig. \ref{fig:swHist} demonstrates the striking spectral similarity between the two graphs, both in the extreme eigenvalues and in the bulk eigenvalues.

The alignment of the bulk eigenvalues from the observed set of graphs and the bulk eigenvalues of $\widehat G_{\tilde N, \mu_{\rho_n f}}^*$ showcases that the graph structure may be entirely determined by the $c$ largest eigenvalues. This has been well know to be true of stochastic block models and in Fig. \ref{fig:swHist} we see evidence that the Small World ensemble might also be characterized by its largest eigenvalues.
\begin{figure}[H]
  \centerline{
    \includegraphics[clip, trim = 0cm 0cm 0cm 0cm, width=\textwidth, height = 0.42\textwidth]{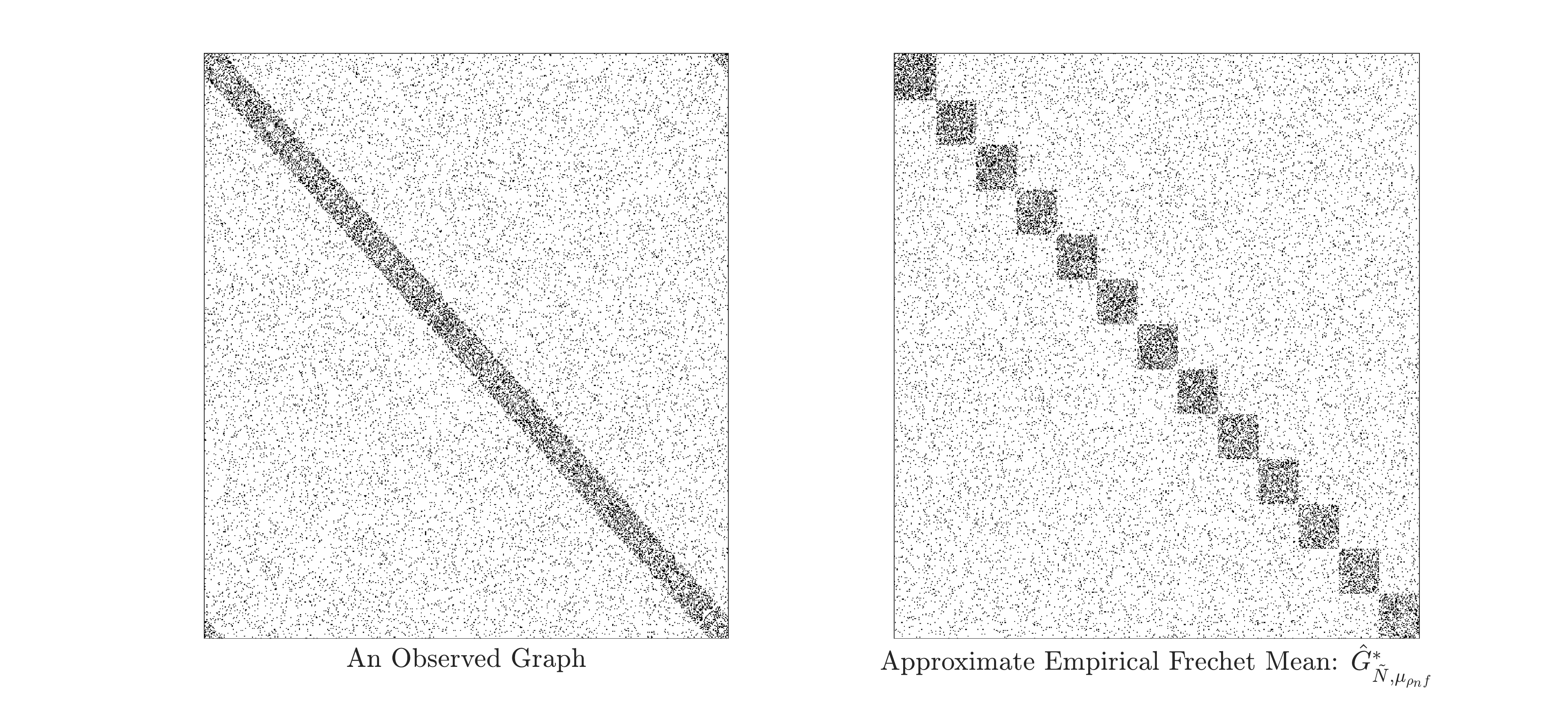}}
  \caption{\textit{Visualization of a graph in $M_2$ and the approximate sample Fr\'echet mean of $M_2$, $\hat G_{\tilde N, \mu_{\rho_n f}}^*$}}
  \label{fig:swFM}
\end{figure}
\begin{figure}[H]
  \centerline{
    \includegraphics[clip, trim = 0cm 0cm 0cm 0cm, width=\textwidth, height = 0.42\textwidth]{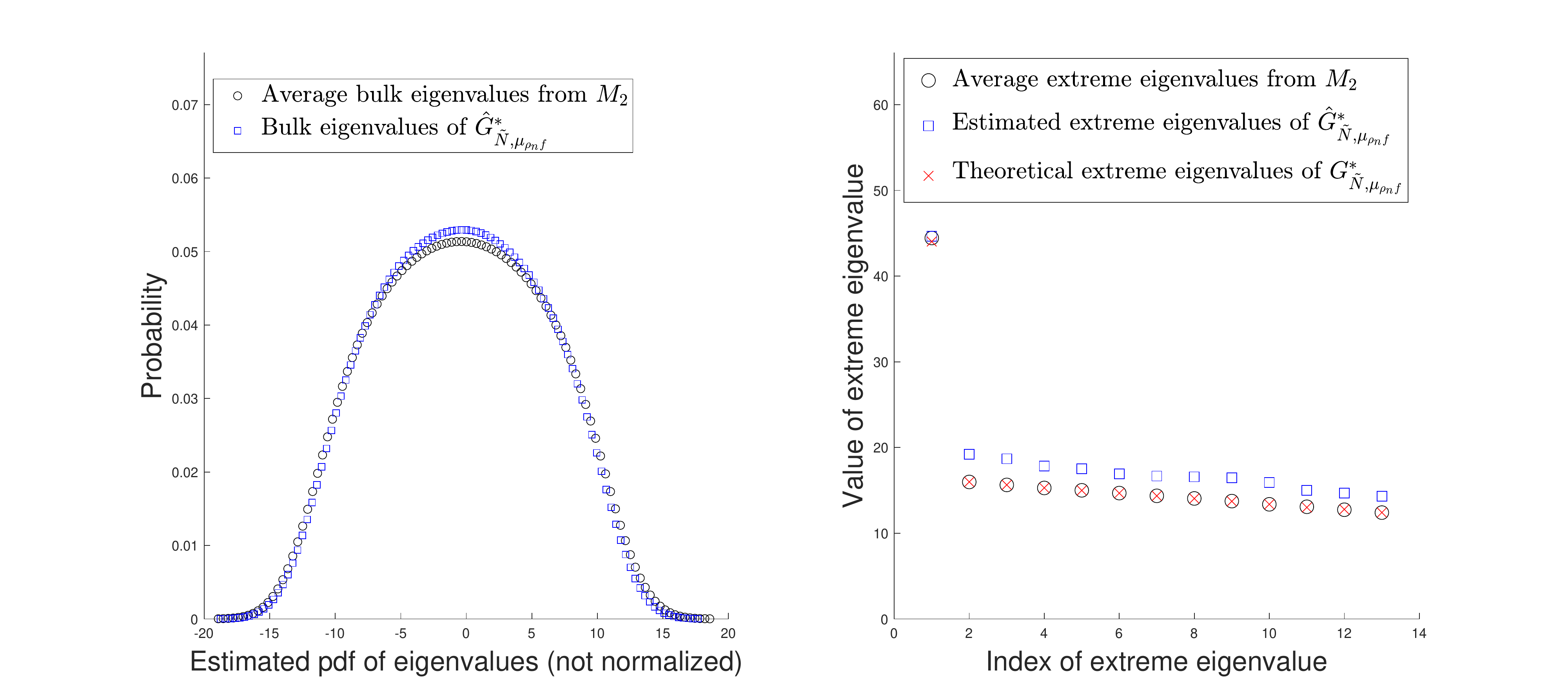}}
  \caption{\textit{\textbf{Left}: The average distribution of bulk eigenvalues from $M_2$ (black). The distribution of bulk eigenvalues of the approximate sample Fr\'echet mean $\widehat G_{\tilde N,\mu_{\rho_n f}}^*$ (blue). \textbf{Right}: The average extreme eigenvalues from $M_2$ (black). The expected extreme eigenvalues of $G_{\tilde N, \mu_{\rho_n f}}^*$ (red). The extreme eigenvalues of $\widehat G_{\tilde N,\mu_{\rho_n f}}^*$ (blue).}}
  \label{fig:swHist}
\end{figure}
\subsection{Variable community size approximate sample Fr\'echet mean} 
\label{subsec:VarComSize}
The probability measure in this section is associated with a variable community sized stochastic block model. The parameters for the stochastic block model are  $\bp = [0.4, 0.5, 0.6, 0.3, 0.37, 0.65,0,...]$, $Q_{ij} = 0.08$ for all $i \neq j$, $\bs = [\frac{160}{600}, \frac{100}{600}, \frac{60}{600}, \frac{120}{600}, \frac{85}{600}, \frac{75}{600},0,...]$. Fig. \ref{fig:varcomFM} again depicts a visual comparison between a graph from $M_3$ and the approximate sample Fr\'echet mean graph $\widehat G_{\tilde N,\mu_{\rho_n f}}^*$.

In spite of the visual difference between the adjacency matrices of a graph from the set $M_3$ and the graph $\widehat G_{\tilde N,\mu_{\rho_n f}}^*$ respectively (see Fig. \ref{fig:varcomFM}), we again see a striking similarity between the eigenvalues (see Fig. \ref{fig:varcomHist}).
\begin{figure}[H]
  \centerline{
    \includegraphics[clip, trim = 0cm 0cm 0cm 0cm, width=\textwidth, height = 0.42\textwidth]{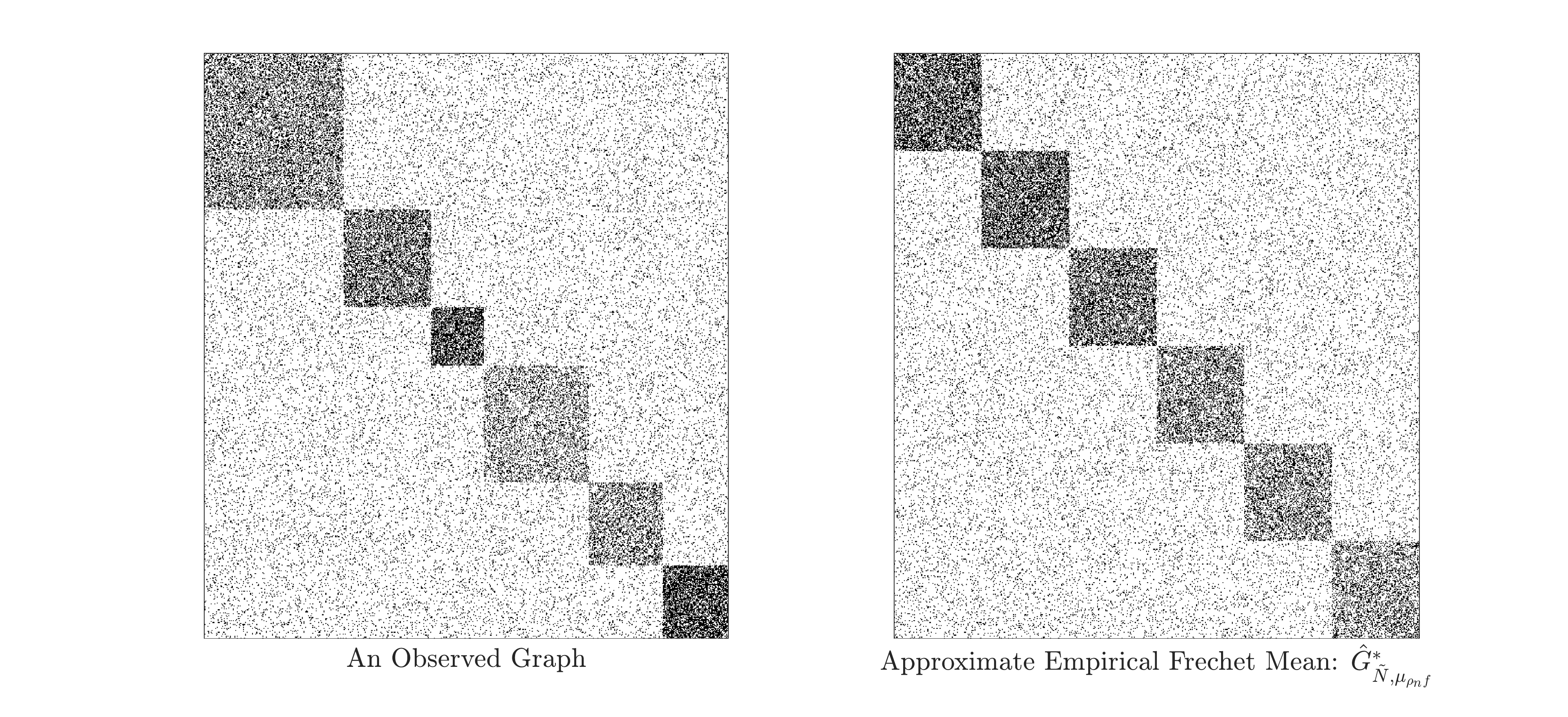}}
  \caption{\textit{Visualization of a graph in $M_2$ and the approximate sample Fr\'echet mean of $M_2$, $\hat G_{\tilde N,\mu_{\rho_n f}}^*$}}
  \label{fig:varcomFM}
\end{figure}
\begin{figure}[H]
  \centerline{
    \includegraphics[clip, trim = 0cm 0cm 0cm 0cm, width=\textwidth, height = 0.42\textwidth]{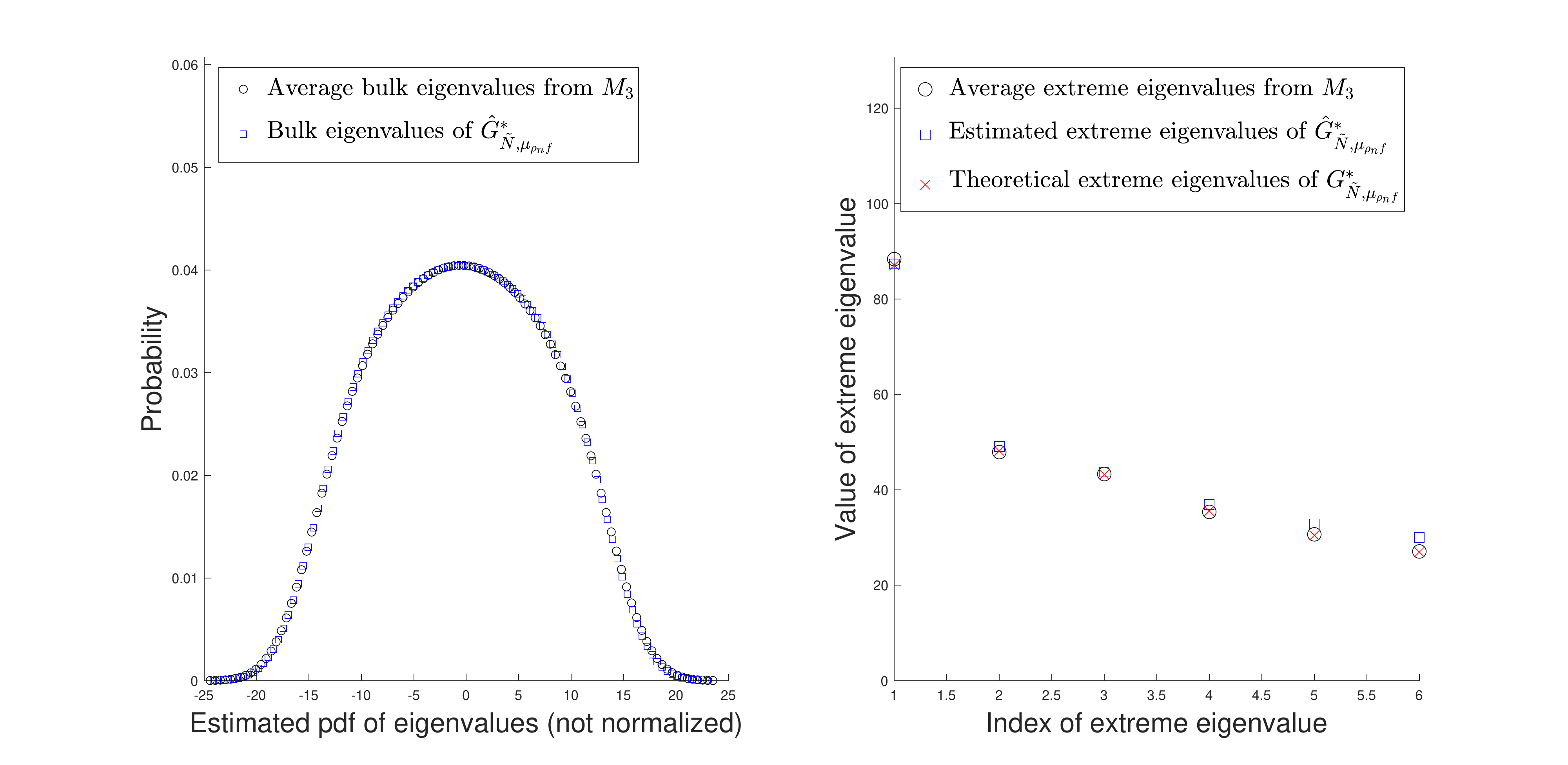}}
  \caption{\textit{\textbf{Left}: The average distribution of bulk eigenvalues from $M_3$ (black). The distribution of bulk eigenvalues of the approximate sample Fr\'echet mean $\hat G_{\tilde N,\mu_{\rho_n f}}^*$ (blue). \textbf{Right}: The average extreme eigenvalues from $M_3$ (black). The expected extreme eigenvalues of $G_{\tilde N, \mu_{\rho_n f}}^*$ (red). The extreme eigenvalues of $\hat G_{\tilde N,\mu_{\rho_n f}}^*$ (blue).}}
  \label{fig:varcomHist}
\end{figure}

Note the similarity in the spectra between the graphs despite the obvious difference in the geometry vectors for each stochastic block model. It is for this reason precisely that we are allowed to choose the geometry vector when searching for the canonical stochastic block model kernel that solves equation (\ref{eqn:DetF}) as mentioned in remark \ref{rem:Geom}.
\section{Application to Graph Valued Regression}
\label{sec:Reg}
In this section we provide an application of the computation of the sample Fr\'echet mean: the construction of a regression
function in the context where we observe a graph-valued random variable that depends on a real-valued random variable. Our approach is based on the theory developed in \cite{PM19} where we replace the computation of the sample Fr\'echet mean with our algorithm. We briefly recall the framework of \cite{PM19} using our notation. We consider the following scenario. Let $\mu \in \cM\left( \cG \right)$, and let $T$ be a random variable with probability density $\prob{t}[T]$. We consider the 
random variable formed by the pair $G$ and $T$, distributed with the joint distribution formed by the product $\mu \times
\prob{t}[T]$.  We wish to compute the regression function
\begin{equation}
  \E{G \vert T = t}.
\end{equation}
The authors in \cite{PM19} propose to compute the following regression function
\begin{align}
  m(t) = \argmin{G \in \cG}\E{s(T,t)d^2(G,G_{\mu})}[\mu \times \prob{t}[T]], \label{eqn:FR}
\end{align}
where the expectation in \eqref{eqn:FR} is computed jointly over $G_\mu$ distributed according to $\mu$, and $T$, distributed
according to $\prob{t}[T]$, and the bilinear form $s(T,t)$ is defined by
\begin{equation}
  s(T,t) = 1 + (T - \E{T}) \left[ \var{T} \right] ^{-1}(t-\E{T}).
\end{equation}
The bilinear form, $s(T,t)$, plays the role of a kernel, returning the location of $t$ with respect to the location, $\E{T}$, and
scale, $\var{T}$, of $T$. The regression function $m(t)$ returns a kernel estimate of the linear regression function by summing
over all the possible pairs $(G_\mu,T)$.

The sample estimate of equation (\ref{eqn:FR}) is the natural estimate where each unknown term is replaced with the sample alternative as
\begin{align}
  \hat{m}(t) = \argmin{G \in \cG} \sum_{k=1}^N s_{k,N}(t)d^2(G,\Gk) \label{eqn:EFR}\\
  s_{k,N}(t) = 1 + (t_k - \bar{T})\hat{V}(t - \bar{T}).
\end{align}

Here we have used $\bar{T}$ and $\hat{V}$ as the sample estimate of the mean and variance of $T$. The objective in
(\ref{eqn:EFR}) can be interpreted as a weighted sample Fr\'echet mean with weight function $s_{k,N}(t)$. Assume for all $t$ that the graph, $\hat m(t)$, satisfies the conditions for Theorem \ref{thm:SpecSimLargeGraphs}. This implies the existence of a sequence of stochastic block model kernels depending on $t$, $\mu_{\rho_n f;t}$, such that, for sufficiently large $n$,
\begin{equation}
  \lim_{N \to \infty} d_{A_c}(\hat{m}(t), G_{N,\mu_{\rho_n f; t}}^*) < \epsilon \quad a.s. \label{eqn:regApprox}
\end{equation}
where $G_{N,\mu_{\rho_n f; t}}^*$ denotes the sample Fr\'echet mean of $\lbrace \Gk_t \rbrace_{k=1}^N$, an iid sample distributed according to $\mu_{\rho_n f;t}$. For each $t$ we may compute $G_{N,\mu_{\rho_n f; t}}^*$ using Alg. \ref{alg:DetFM_C} as an approximation to the graph $\hat{m}(t)$. 
\subsection{Experimental validation for graph valued regression}
We validate the computation of the regression with numerical simulation. We first generate a synthetic data set of graphs by allowing the parameters of the stochastic block model to vary with time. For simplicity we hold the nonzero entries of $\bQ$ and $\bs$ fixed but allow $\bp$ to vary for $t\in[0,1]$ as
\begin{equation}
  \rho_n \bm p(t) = \begin{bmatrix} 0.1 + 0.1 t\\ 0.2 + 0.15t\\ 0.35 + 0.2t \\ 0 \\ \vdots \end{bmatrix}, \bs(t) = \begin{bmatrix}1/3 \\ 1/3 \\ 1/3 \\ 0 \\ \vdots \end{bmatrix}, \rho_n q = 0.08.
\end{equation}
For $T \sim unif(0,1)$, the distribution over $\cG$ is given as $\mu_{\rho_n f; T}$ where $f(x,y; \bm p (T),\bQ, \bs)$ is a canonical stochastic block model kernel. For each sample from $unif(0,1)$ there is a corresponding sample from the stochastic block model. By construction we know the number of communities in the observed graphs will be constant at $c = 3$ dictating the number of non-zero entries of $\bm p$ we allow to vary when searching for the stochastic block model kernel in equation (\ref{eqn:regApprox}).

We take $n = 600$ and $N = 30$ samples for the sample set $M = \lbrace (t_k, \Gk) \rbrace_{k=1}^{30}$ in the experiment and approximate the value of $\hat{m}(t)$ at six different times. For $t' \in \{0,0.2,0.4,0.6,0.8,1\}$ we compute $\widehat{G}_{\tilde N, \mu_{\rho_n f; t'}}^* $ using Alg. \ref{alg:DetFM_C}. 

In an effort of visualization, since we are unable to plot a graph $G$ on the y-axis, we plot in Fig. \ref{fig:FR-e} the largest three eigenvalues of the adjacency matrices of graphs in $M$ (marked by a $\bullet$) and the largest three eigenvalues of $\widehat{G}_{\tilde N, \mu_{\rho_n f; t'}}^*$ (marked by an $\times$). A vertical line of points in Fig. \ref{fig:FR-e} identifies the largest three eigenvalues of a single graph.
\begin{figure}[H]
  \centering
  \includegraphics[clip, trim = 0cm 0.1cm 0cm 1.1cm, width=\linewidth]{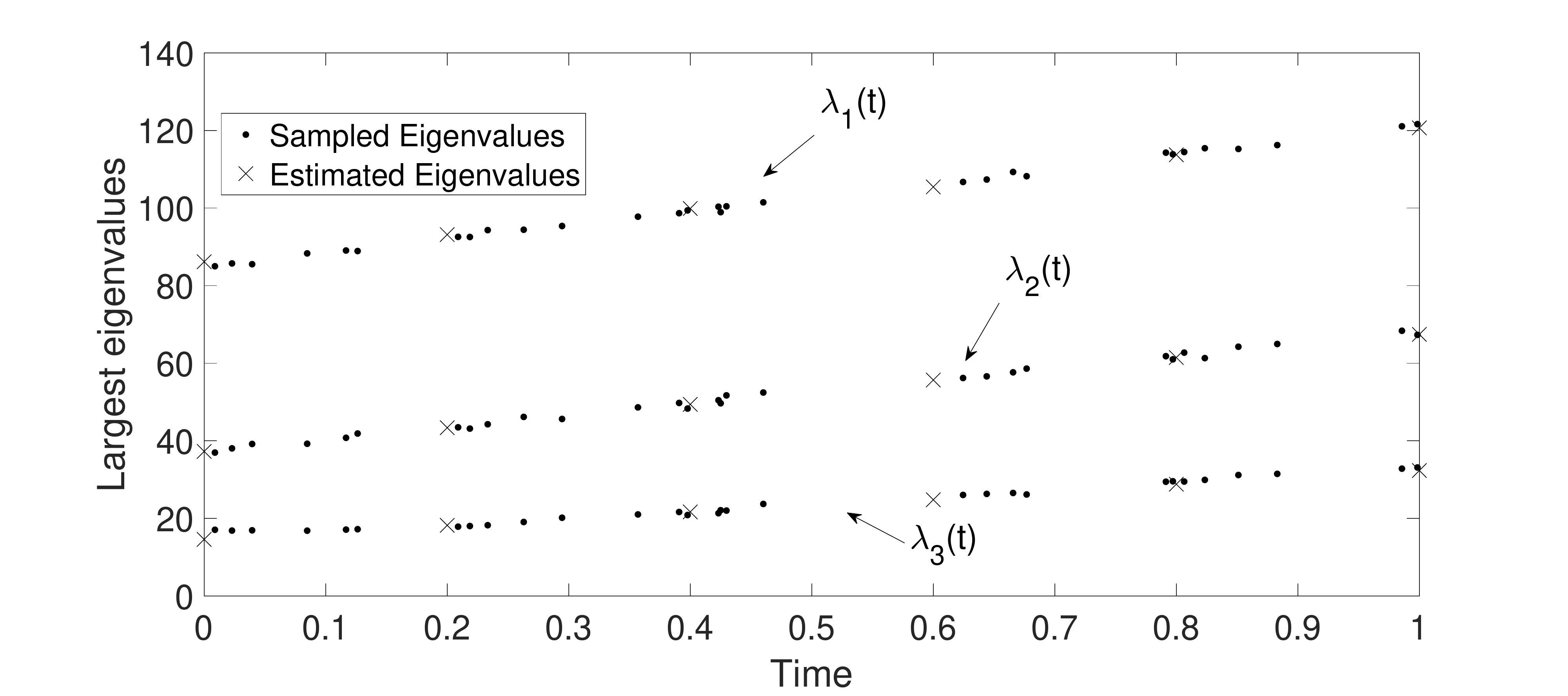}
  \caption{Recovered eigenvalues}
  \label{fig:FR-e}
\end{figure}
The most notable part of Fig. \ref{fig:FR-e} is the construction of a graph that fits the linear regression of each of the largest $c$ eigenvalues simultaneously. To our knowledge, this is the first graph valued linear regression line with respect to a spectral distance. 

\section{Conclusion.}
In the area of statistical analysis of graph-valued data, determining an average graph is a point of priority among
researchers. Throughout this paper, we have shown that when considering the metric $d_{A_c}$ it is possible to
determine an approximation to the sample Fr\'echet mean.

How this approximate sample Fr\'echet mean is utilized is up to the discretion of the researcher. In Section \ref{sec:Reg} we
explore one motivating idea that utilizes the Fr\'echet mean, termed Fr\'echet regression in the work in \cite{PM19}. This is
but one example of the utility of the Fr\'echet mean graph, another interesting application of this graph is to further push the
work in \cite{LOW20} which introduces a centered random graph model to capture the variance of a set of observations around a
mean graph.


\newpage
\appendix
\begin{center}
  \textbf{Appendix}\\
\end{center}

We split the appendix into four sections. \ref{app:Classic} establishes a few classic results that we refer to in our proofs. The proof of our primary contribution, Theorem \ref{thm:SpecSimLargeGraphs}, is contained in \ref{app:SpecSimLargeGraphs}. Within this appendix we also prove Theorem \ref{thm:AlmostSure} and Theorem \ref{thm:EstLargeEigs} since these are necessary results for our proof of Theorem \ref{thm:SpecSimLargeGraphs}. \ref{app:GeomEigsFM} and \ref{app:HighProbStat} are short appendices in which we prove Theorems \ref{thm:GeomEigsFM} and \ref{thm:HighProbStat} respectively.
\begin{center}
  \section{Classic Results}
  \label{app:Classic}
\end{center}
\begin{theorem}[Weyl-Lidskii]\hfill \\
  \label{thm:SpecInc}
  Let $\bm H$ be a self-adjoint operator on a Hilbert space $\mathcal{H}$. Let $\bA$ be a bounded operator on $\mathcal{H}$ Let $\sigma(\bm H)$ and $\sigma(\bm H + \bA)$ denote the spectra of $\bm H$ and $(\bm H + \bA)$ respectively. Then
  \begin{equation}
    \sigma(\bm H + \bA) \subset \left \lbrace \lambda : dist(\lambda,\sigma(\bm H)) \leq ||\bA|| \right \rbrace
  \end{equation}
  where $||\bA||$ denotes the operator norm of $\bA$.
\end{theorem}
\begin{proof}
  These are standard bounds that can be found in many good books on matrix perturbation theory (e.g., \cite{SS90}).
\end{proof}

\noindent Let $P_n$ be probabilities on the Borel $\sigma$-field of $\R^c$ and suppose $P_n \to P$ weakly.
\begin{theorem}[Finite Dimensional Convergence in Distribution]\hfill \\
  \label{thm:ConvDistEq}
  Let $F_n(\bx):= P_n((-\infty, x_1] \times ... \times (-\infty, x_c])$ and $F(\bx):= P((-\infty, x_1] \times ... \times (-\infty, x_c])$ for any $\bx \in \R^c$. Then $F_n(\bx) \to F(\bx)$ as $n \to \infty$ for every point of continuity $\bx$ of $F(\bx)$.
\end{theorem}
\begin{proof}
  This is a standard equivalence for convergence in distribution found in e.g. \cite{BW16}.
\end{proof}

\begin{center}
  \section{Proof of Theorem \ref{thm:SpecSimLargeGraphs}}
  \label{app:SpecSimLargeGraphs}
\end{center}
Theorem \ref{thm:SpecSimLargeGraphs} constitutes the main theoretical contribution of this paper. The proof involves a few steps which we outline below at a high level.

\begin{enumerate}
\item Given $G$ with adjacency matrix $\bA$ we compute $\sigma_c(\bA)$ and show that we may construct a canonical stochastic block model kernel, $f$, where the linear integral operation, $L_f$, defined by 
  \begin{equation}
    L_f(t) = \int_0^1 f(x,y;\bp, \bQ, \bs) t(y) dy
  \end{equation}
  has eigenvalues such that $\lambda_i(L_f) = \frac{\lambda_i(\bA)}{n \rho_n}$. The normalization by $n \rho_n$ is understood because we will be scaling the eigenvalues to that of an $n \times n$ matrix, and the constant $\rho_n$ which is due to the definition of the kernel probability measure (see Definition \ref{def:KernProbMeas}).
\item We then show that for each $1 \leq i \leq c$, $\left|n \rho_n \lambda_i(L_f) - \E{\lambda_i(\bA_{\mu_{\rho_n f}})}\right| = \mathcal{O}(\sqrt{\rho_n})$. Recall that since $\rho_n \to 0$ we will have $\lambda_i(L_f) \to \E{\frac{1}{n \rho_n}\lambda_i(\bA_{\mu_{\rho_n f}})} $.
\item All that is left is to show that for an iid sample of graphs $\lbrace \Gk \rbrace_{k=1}^N$ distributed according to $\mu_{\rho_n f}$, the sample Fr\'echet mean, $G_N^*$, with adjacency matrix $\bA_N^*$, satisfies that $\forall \epsilon > 0$, there exists an $n^* \in \N$ such that $\forall n > n^*$,  
  $$\lim_{N \to \infty} ||\sigma_c(\bA_N^*) - \E{\sigma_c(\bA_{\mu_{\rho_n f}})}||_2 < \epsilon \quad a.s.$$
  The strategy in this step is to show the existence of a different graph $G'$ with adjacency matrix $\bA'$ whose eigenvalues are close to $\E{\sigma_c(\bA_{\mu_{\rho_n f}})}$. The existence of the graph $G'$ will allow us to bound the distance between the eigenvalues of $\bA_N^*$ and $\E{\sigma_c(\bA_{\mu_{\rho_n f}})}$. This is because $G_N^*$ satisfies the minimization procedure in (\ref{eqn:EFM}) which we will show is equivalent to finding the graph closest to $\E{\sigma_c(\bA_{\mu_{\rho_n f}})}$. 
\item The final step is to show 
  $$|\lambda_i(\bA) - \E{\lambda_i(\bA_{\mu_{\rho_n f}})} + \E{\lambda_i(\bA_{\mu_{\rho_n f}})} - \lambda_i(\bA_N^*)| < \epsilon \quad a.s.$$
  for each $1\leq i \leq c$ which comes as a direct consequence of steps 1,2 and 3.
\end{enumerate}
We now proceed with our proof.\\

\noindent \textbf{Step 1: Constructing a stochastic block model kernel} \hfill \\

Let $G \in \cG$ with adjacency matrix $\bA$ such that $n^{-2/3} \ll \rho_n \ll 1$. Assume that 
\begin{equation}
  0 \preccurlyeq \sigma_c(\bA)
\end{equation}
and for every $1 \leq i \neq j \leq c$, $\lambda_i \neq \lambda_j$.
\begin{Lemma}
  \label{lem:SBMCorrectEigs}
  Let $\bm \theta \in \R^c$ such that $0 \preccurlyeq \bm \theta$. Let $\bs$ be a fixed geometry vector for a canonical stochastic block model kernel with $c$ non-zero entries. There exists a canonical stochastic block model kernel $f(x,y;\bp,\bQ,\bs)$ with $\bQ = 0$ defining the integral operator $L_f : L^2([0,1]) \mapsto L^2([0,1])$ by
  \begin{equation}
    L_{f}(t) = \int_0^1 f(x,y;\bp, \bQ, \bs) t(y) dy
  \end{equation}
  which satisfies
  \begin{equation}
    \lambda_i(L_f) = \theta_i
  \end{equation}
\end{Lemma}
\begin{lproof}
  The proof is rather straightforward, we simply construct the equivalent of a block diagonal matrix. The blocks are determined by the geometry vector $\bs$ which we are free to choose within the constraints that $||\bs ||_1 = 1$, $\bs$ is non-increasing, non-negative, and has $c$ non-zero entries. For $1\leq i \leq c$, let $S_i = \sum_{j=1}^i s_j$ and define the intervals $\cI_i = [S_{i-1}, S_{i})$. Note that $S_0 = 0.$ Define the function
    $$f (x,y) = \begin{cases} \frac{\theta_i}{s_i} \quad \text{if } (x,y) \in I_i \times I_i\\ 0 \quad \text{else.} \end{cases}$$
    Define the linear integral operator $L_{f}(t) = \int_0^1 f(x,y) t(y) dy$. The eigenfunctions for $L_{f}$ are
    $$r_i(x) = \begin{cases} \frac{1}{\sqrt{s_i}} \quad \text{if } x \in \cI_i \\ 0 \quad \text{else.} \end{cases}$$ which we show in the following computations. We can compute the eigenvalues of $L_{f}$ as
    \begin{align}
      L_{f}(r_i(x)) &= \int_0^1 f(x,y) r_i(y) dy\\
      & = \begin{cases} \int_{\cI_i} \frac{\theta_i}{s_i} \frac{1}{\sqrt{s_i}} dy \quad x \in \cI_i \\ 0 \quad \text{else}\end{cases}\\
      &= \begin{cases} \frac{\theta_i}{s_i} \frac{1}{\sqrt{s_i}} s_i \quad x \in \cI_i \\ 0 \quad \text{else}\end{cases}\\
      &= \begin{cases} \theta_i \frac{1}{\sqrt{s_i}} \quad x \in \cI_i \\ 0 \quad \text{else}\end{cases}\\
      &= \theta_i r_i(x).
    \end{align}
    Next we verify that $||r_i||_2 = 1$. 
    \begin{align}
      \int_0^1 r_i(x)^2 dx &= \int_{\cI_i} \frac{1}{s_i} dx\\
      &= \frac{s_i}{s_i}\\
      &= 1.
    \end{align}
    Therefore $r_i(x)$ is an eigenfunction of $L_{f}$ with eigenvalue $\theta_i$. At this point we note that $f$ is a stochastic block model kernel with $q_{ij} = 0$ for all $i,j$.
\end{lproof}
By taking $\bm \theta = \frac{\sigma_c(\bA)}{n \rho_n}$ in Lemma \ref{lem:SBMCorrectEigs} we will have accomplished step 1.\\

\noindent \textbf{Step 2: Estimating the expected eigenvalues} \hfill \\

We first introduce a recent theorem from \cite{CCH20} that discusses an estimate of the expected eigenvalues of an inhomogeneous \ER $\text{ }$random graph. Let $\mu_{\rho_n f} \in \cM(\cG)$ be a kernel probability measure with kernel $f$. Let $L_f$ be the linear integral operator with the same kernel function, $f$. Assume $L_f$ has a finite rank of $c$. Denote the eigenvalues and eigenfunctions of $L_f$ as $\theta_i$ and $r_i(x)$ respectively where for each $i = 1,...,c$, $r_i(x)$ is assumed to be piecewise Lipschitz with finitely many discontinuities.
\begin{theorem}[Chakrabarty, Chakraborty, Hazra 2020] \hfill \\
  \label{thm:EstLargeEigs-CCH}
  For every $1 \leq i \leq c$,
  \begin{equation}
    \E{\lambda_i(\bA_{\mu_{\rho_n f}})} = \lambda_i(\bm B) + \mathcal{O}(\sqrt \rho_n + \frac{1}{n \rho_n}), \label{eqn:EstLargeEigs-CCH}
  \end{equation}
  where $\bm B$ is a $c \times c$ symmetric deterministic matrix defined by
  $$b_{j,l} = \sqrt{\theta_j \theta_l} n \rho_n \bm e^T_j \bm e_l + \mathcal{O}(\frac{1}{n \rho_n}),$$
  and $\bm e_j$ is a vector with entries $\bm e_j(k) = \frac{1}{\sqrt n} r_j(\frac k n)$ for $1 \leq j \leq c$.
\end{theorem}
\begin{proof}
  The proof is in \cite{CCH20}.
\end{proof}
The authors in \cite{CCH20} require that the eigenfunctions be Lipschitz but, as is made clear from their proof, this requirement can be relaxed to include piecewise Lipschitz functions with no adjustments to their proof. The eigenfunctions of integral operators with stochastic block model kernels are therefore within the scope of this theorem. In this section we only analyze the result as it applies to canonical stochastic block model kernels.

Note that the estimate provided in Theorem \ref{thm:EstLargeEigs-CCH} is dependent on the eigenfunctions of $L_f$. We now prove that omitting the contribution of the eigenfunctions, $r_i(x)$, to the estimate in equation (\ref{eqn:EstLargeEigs-CCH}) still results in an estimate whose error decays like $\cO(\sqrt{\rho_n})$. This will conclude step 2 and serve as a proof for Theorem \ref{thm:EstLargeEigs} which we restate for convenience below.
\begin{theorem}[Theorem \ref{thm:EstLargeEigs} from the main paper]\hfill\\
  \label{thm:EstLargeEigs-app}
  For every $1 \leq i \leq c$,
  \begin{equation}
    \E{\lambda_i(\bA_{\mu_{\rho_n f}})} = \theta_i n \rho_n + \mathcal{O}(\sqrt \rho_n). \label{eqn:EstLargeEigs}
  \end{equation}
\end{theorem}
\begin{proof}
  Let $\bm B$ be the $c \times c$ symmetric deterministic matrix in Theorem \ref{thm:EstLargeEigs-CCH} defined by
  $$b_{j,l} = \sqrt{\theta_j \theta_l} n \rho_n \bm e^T_j \bm e_l + \mathcal{O}(\frac{1}{n \rho_n}),$$
  where $\bm e_j$ is a vector with entries $\bm e_j(k) = \frac{1}{\sqrt n} r_j(\frac k n)$ for $1 \leq j \leq c$.
  We will show that the off diagonal entries of $\bm B$ tend to zero at the rate $\mathcal{O}(\frac 1 n)$ and the diagonal entries, $B_{l,l}$, are given by $\theta_l n \rho_n + \mathcal{O}(\frac 1 n)$. We then show via Weyl-Lidskii that $\lambda_i(\bm B) = \theta_l n \rho_n + \mathcal{O}(\rho_n)$. We begin by examining the off diagonal entries of $\bm B$.

  Given $b_{j,l} = \sqrt{\theta_j \theta_l} n \rho_n \bm e^T_j \bm e_l + \mathcal{O}(\frac{1}{n \rho_n})$. Observe that for all $j \neq l$,  we have
  \begin{align}
    0 &= \int_0^1 r_j(x) r_l(x) dx = \lim_{n \to \infty} \sum_{m = 0}^n r_j(\frac m n) r_l(\frac m n) \frac 1 n\\
    &= \lim_{n \to \infty} \sum_{m=0}^n \bm e_j(m) \bm e_l(m) = \lim_{n \to \infty} \bm e_j^T \bm e_l.
  \end{align}
  The interpretation is that $\bm e_j^T \bm e_l$ is an approximation to the integral using the right end points of the intervals. Denote by $R(n) = |\int_0^1 r_j(x) r_l(x) dx - \sum_{m = 0}^n r_j(\frac m n) r_l(\frac m n) \frac 1 n|$ the error in the right end point approximation of the integral. Then $R(n) = |\sum_{m = 0}^n r_j(\frac m n) r_l(\frac m n) \frac 1 n|$ since the functions $r_j(x)$ and $r_l(x)$ are orthogonal. The convergence rate for the right end point rule is $R(n) = \mathcal{O}(\frac 1 n)$ since $r_i(x)$ is piecewise Lipschitz on a bounded interval. Consequently,
  $$n \rho_n \bm e_j^T \bm e_l = n \rho_n R(n) = n \rho_n \mathcal{O}(\frac 1 n) = \mathcal{O}(\rho_n).$$
  So for all $j\neq l$, $b_{j,l} = \mathcal{O}(\rho_n)$.\\
  
  We next must show that the diagonal elements of $\bm B$ are $\theta_i n \rho_n + \mathcal{O}(\rho_n)$. We will use the same observations as earlier. For $j = l$, note that $R(n) = |\int_0^1 r_l(x) r_l(x) dx - \sum_{m = 0}^n r_l(\frac m n) r_l(\frac m n) \frac 1 n| = |1 - \sum_{m = 0}^n r_l(\frac m n) r_l(\frac m n) \frac 1 n|$. Since each $r_i(x)$ is Lipschitz on a bounded interval, we have $R(n) = \mathcal{O}(\frac 1 n)$. As a consequence, $\bm e_l^T \bm e_l = 1 + \mathcal{O}(\frac 1 n)$. Therefore the term 
  $$n \rho_n \bm e_l^T \bm e_l = n \rho_n  \left(1 + \mathcal{O}(\frac 1 n) \right)= n \rho_n + \mathcal{O}(\rho_n).$$
  Thus the diagonal terms of $\bm B$ are 
  $$b_{l,l} = \theta_l n \rho_n + \mathcal{O}(\rho_n).$$
  We next must show that $\lambda_l(\bm B) = \theta_l n \rho_n + \mathcal{O}(\rho_n)$. Let $\tilde{\bm B}$ be a matrix with diagonal elements $\tilde b_{i,i} = b_{i,i}$ and that for all $i \neq j$, $\tilde b_{i,j} = 0$. Note that we have shown all off-diagonal elements of $\bm B$ are of order $\mathcal{O}(\rho_n)$. Therefore we may write 
  $$\bm B = \tilde{\bm B }+ \bm M,$$
  where $M_{i,j} = \mathcal{O}(\rho_n)$. Note that since $M$ is a $c \times c$ matrix we have $$||\bm M||_{F} = \mathcal{O}(\rho_n).$$ 
  We can conclude using Weyl-Lidskii's theorem that 
  $$|\lambda_i(\tilde{\bm B}) - \lambda_i(\bm B)| = |\theta_i n \rho_n - \lambda_i(\bm B)| =  \mathcal{O}( \rho_n).$$
  Substituting in equation (\ref{eqn:EstLargeEigs-CCH}),
  \begin{align}
    \E{\lambda_i(\bA_{\mu_{\rho_n f}})} &= \lambda_i(\bm B) + \mathcal{O}(\sqrt \rho_n + \frac{1}{n \rho_n})\\
    &= \theta_i n \rho_n + \cO(\rho_n) + \mathcal{O}(\sqrt \rho_n + \frac{1}{n \rho_n})\\
    &= \theta_i n \rho_n + \mathcal{O}(\sqrt \rho_n)
  \end{align}
  where the last equality holds since the slowest decaying term of $\cO(\sqrt \rho_n + \rho_n + \frac{1}{n \rho_n})$ is $\cO(\sqrt \rho_n)$. \qed
\end{proof}
Theorem \ref{thm:EstLargeEigs} provides a slightly worse estimate of the expected eigenvalues, $\E{\lambda_i(\bA_{\mu_{\rho_n f}})}$, than Theorem \ref{thm:EstLargeEigs-CCH} though it should be noted that the error term in both is of the same order, $\sqrt{\rho_n}$, since we have $n^{-2/3} \ll \rho_n$.

Theorem \ref{thm:EstLargeEigs} gives a method to estimate the expected eigenvalues of the stochastic block model kernel probability measure defined in Lemma \ref{lem:SBMCorrectEigs}. In particular, Theorem \ref{thm:EstLargeEigs} shows that for any $\mu_{\rho_n f}$ and for each $1 \leq i \leq c$,
\begin{equation}
  |\E{\lambda_i(\bA_{\mu_{\rho_n f}})} - n \rho_n \theta_i| =  \mathcal{O}(\sqrt{\rho_n}).
\end{equation}

\noindent \textbf{Step 3: Eigenvalues of the sample Fr\'echet mean adjacency matrix are nearly the expected eigenvalues} \hfill \\

Step 3 of our proof is arguably the most interesting. As was stated in the outline, given a stochastic block model kernel probability measure, $\mu_{\rho_n f}$, we show the existence of a graph, $G'$, whose eigenvalues are close to the $\E{\sigma_c(\bA_{\mu_{\rho_n f}})}$. By showing the existence of this graph, we will be able to provide upper bounds on the distance between the eigenvalues of the adjacency matrix of the sample Fr\'echet mean graph, $\sigma_c(\bA_N^*)$ and $\E{\sigma_c(\bA_{\mu_{\rho_n f}})}$.

To prove there exists a graph, $G'$, whose adjacency matrix has eigenvalues close to $\E{\sigma_c(\bA_{\mu_{\rho_n f}})}$, we will show that there exists a positive constant $C$ such that 
\begin{equation}
  0 < C < P\left(||\sigma_c(\bA_{\mu_{\rho_n f}}) - \E{\sigma_c(\bA_{\mu_{\rho_n f}})}||_2 < \epsilon \right).
\end{equation}
Since the probability measure of the set of graphs that satisfy $||\sigma_c(\bA_{\mu_{\rho_n f}}) - \E{\sigma_c(\bA_{\mu_{\rho_n f}})}||_2 < \epsilon$ is strictly positive, this implies the existence of at least one graph, $G'$, that satisfies the inequality. To prove that the probability is nonzero we reference Theorem 2.3 from \cite{CCH20} on the convergence in distribution of the extreme eigenvalues of inhomogeneous \ER $\text{ }$ random graphs which we state below.
\begin{theorem}[Chakrabarty, Chakraborty, and Hazra 2020] \hfill \\
  \label{thm:ConvDist}
  For every $1 \leq i \leq c$ ,
  \begin{equation}
    \rho_n^{-1/2}(\lambda_i(\bA_{\mu_{\rho_n f}}) - \E{\lambda_i(\bA_{\mu_{\rho_n f}})}) \overset{d}{\to} (Z_i : 1 \leq i \leq c), \label{eqn:ConvDist}
  \end{equation}
  where the right hand side is a multivariate normal random vector in $\R^c$, with mean zero and
  \begin{equation}
    Cov(Z_i, Z_j) = 2 \int_0^1 \int_0^1 r_i(x) r_i(y) r_j(x) r_j(y) f(x,y) dx dy, \label{eqn:CovMat}
  \end{equation}	
  for all $1 \leq i,j \leq c$.
\end{theorem}
\begin{proof}
  Theorem \ref{thm:ConvDist} is Theorem 2.3 in \cite{CCH20}.
\end{proof}
We first acknowledge that there exists a very similar theorem to the above in \cite{ACT21} with the difference being the centering of the limiting distribution about the eigenvalues of the expected adjacency matrix, $\lambda_i(\E{\bA_{\mu_{\rho_n f}}})$, rather than the expected eigenvalues, $\E{\lambda_i(\bA_{\mu_{\rho_n f}})}$. In our case there is no distinction to using either theorem since we have shown in Theorem \ref{thm:EstLargeEigs} that, to first order, $\E{\lambda_i(\bA_{\mu_{\rho_n f}})}$ can be estimated by $\lambda_i(\E{\bA_{\mu_{\rho_n f}}})$.

We consider all the eigenvalues at once by writing $\rho_n^{-1/2}(\sigma_c(\bA) - \E{\sigma_c(\bA_{\mu_{\rho_n f}})})$ rather than analyzing for each $i$. Let $Z$ denote the multivariate normal random vector on the right hand side in equation (\ref{eqn:ConvDist}). An equivalent statement to Theorem \ref{thm:ConvDist} is then
\begin{equation}
  \rho_n^{-1/2}(\sigma_c(\bA_{\mu_{\rho_n f}}) - \E{\sigma_c(\bA_{\mu_{\rho_n f}})}) \overset{d}{\to} Z.
\end{equation}
One characterization of convergence in distribution for finite dimensional random variables is pointwise convergence of the cumulative distribution functions (see Theorem \ref{thm:ConvDistEq}).

Let $P_n$ denote the probabilities for the sequence of random vectors $\rho_n^{-1/2}(\sigma_c(\bA_{\mu_{\rho_n f}}) - \E{\sigma_c(\bA_{\mu_{\rho_n f}})})$ and $P$ be the probability for the multivariate Gaussian random vector $Z$. $\forall \bm z \in \R^c$, define the cumulative distribution function of the random variables $\rho_n^{-1/2}(\sigma_c(\bA_{\mu_{\rho_n f}}) - \E{\sigma_c(\bA_{\mu_{\rho_n f}})})$ and $Z$ respectively as
\begin{align} 
  F_n(\bm z) = P_n\left( (-\infty, z_1] \times ... \times  (-\infty, z_c] \right),\\
    F(\bm z) = P\left( (-\infty, z_1] \times ... \times  (-\infty, z_c] \right).
\end{align}
The convergence in distribution of the random vectors is equivalent to the following: $\forall \bm z \in \R^c$,
\begin{align}
  \lim_{n \to \infty} F_n\left(\bm z\right) = F\left( \bm z\right), \label{eqn:ConvDistEq}
\end{align}
since $F(\bm z)$ is continuous everywhere. For our proof it is easier to work with the probabilities and random vectors directly and so equation (\ref{eqn:ConvDistEq}) takes the form
\begin{equation}
  \lim_{n \to \infty} P_n(\rho_n^{-1/2}(\sigma_c(\bA_{\mu_{\rho_n f}}) - \E{\sigma_c(\bA_{\mu_{\rho_n f}})}) \preccurlyeq \bm z) = P(Z \preccurlyeq \bm z). \label{eqn:ConvDistEqVec}
\end{equation}
We are now ready to state and prove our lemma. Let $\mu_{\rho_n f} \in \cM(\cG)$ be a kernel probability measure with kernel $f$. Let $L_f$ be the linear integral operator with the same kernel $f$. Assume $L_f$ has a finite rank $c$ and denote the eigenvalues and eigenfunctions of $L_f$ as $\theta_i$ and $r_i(x)$ respectively where for each $i = 1,...,c$, $r_i(x)$ is assumed to be piecewise Lipschitz with finitely many discontinuities.

\begin{Lemma}
  \label{lem:GraphNearExpEigs}
  $\forall \epsilon > 0$, $\exists n^* \in \N$ where $\forall n > n^*$, $\exists G' \in \cG$ with adjacency matrix $\bA'$ such that 
  \begin{equation}
    ||\sigma_c(\bA') - \E{\sigma_c(\bA_{\mu_{\rho_n f}})}||_2 < \epsilon. \label{eqn:SpecCloseExpEig}
  \end{equation}
\end{Lemma}

\begin{lproof}
  Let $\epsilon > 0$. Fix $\bm z \in \R^c$ such that $0 \preccurlyeq \bm z$ and 
  \begin{equation}
    0 < C < P(-\bm z \preccurlyeq Z \preccurlyeq \bm z) - 2\epsilon, \label{eqn:ArbProb}
  \end{equation}
  for some $C > 0$. By equation (\ref{eqn:ConvDistEqVec}), there exists $n_1 \in \N$ where for all $n > n_1$, 
  \begin{equation}
    \left| P_n\left( \rho_n^{-1/2} (\sigma_c(\bA_{\mu_{\rho_n f}}) - \E{\sigma_c(\bA_{\mu_{\rho_n f}})}) \preccurlyeq \bm z \right) - P\left( Z \preccurlyeq \bm z \right)\right| < \epsilon \label{eqn:Ptwise1}
  \end{equation}
  Similarly, there exists $n_2 \in \N$ where for all $n > n_2$, 
  \begin{equation}
    \left| P_n\left(  \rho_n^{-1/2} (\sigma_c(\bA_{\mu_{\rho_n f}}) - \E{\sigma_c(\bA_{\mu_{\rho_n f}})})  \preccurlyeq -\bm z\right) - P\left( Z \preccurlyeq -\bm z \right)\right| < \epsilon \label{eqn:Ptwise2}
  \end{equation}
  Furthermore, there exists $n_3 \in \N$ where for all $n > n_3$,
  \begin{align}
    ||\sqrt{\rho_n} \bm z||_2 < \epsilon.
  \end{align}
  Take $n^* = \max(n_1,n_2,n_3)$ and consider the following probability
  \begin{align}
    P_n(-\bm z \preccurlyeq \rho_n^{-1/2} (\sigma_c(\bA_{\mu_{\rho_n f}}) - \E{\sigma_c(\bA_{\mu_{\rho_n f}})}) \preccurlyeq \bm z ) &= P_n(\rho_n^{-1/2} (\sigma_c(\bA_{\mu_{\rho_n f}}) - \E{\sigma_c(\bA_{\mu_{\rho_n f}})}) \preccurlyeq \bm z ) \label{eqn:ProbToBound1}\\
    &- P_n(\rho_n^{-1/2} (\sigma_c(\bA_{\mu_{\rho_n f}}) - \E{\sigma_c(\bA_{\mu_{\rho_n f}})}) \preccurlyeq -\bm z). \label{eqn:ProbToBound2}
  \end{align}
  Now, (\ref{eqn:Ptwise1}) and (\ref{eqn:Ptwise2}) allow us to replace (\ref{eqn:ProbToBound1}) and (\ref{eqn:ProbToBound2}) with the corresponding expression in terms of $\bm z$. We therefore obtain
  \begin{equation}
    \left| P_n\left(-\bm z \preccurlyeq \rho_n^{-1/2} (\sigma_c(\bA_{\mu_{\rho_n f}}) - \E{\sigma_c(\bA_{\mu_{\rho_n f}})}) \preccurlyeq \bm z \right) - P\left(-\bm z \preccurlyeq Z \preccurlyeq \bm z\right) \right| < 2 \epsilon,
  \end{equation}
  from which we obtain
  \begin{align}
    P(-\bm z \preccurlyeq Z \preccurlyeq \bm z) - 2\epsilon &<  P_n(- \rho_n^{1/2} \bm z \preccurlyeq (\sigma_c(\bA_{\mu_{\rho_n f}}) - \E{\sigma_c(\bA_{\mu_{\rho_n f}})}) \preccurlyeq \rho_n^{1/2} \bm z )
  \end{align}
  Since $0 < C < P(-\bm z \preccurlyeq Z \preccurlyeq \bm z) - 2 \epsilon$ we have  
  \begin{align}
    C &< P_n(- \rho_n^{1/2} \bm z \preccurlyeq (\sigma_c(\bA_{\mu_{\rho_n f}}) - \E{\sigma_c(\bA_{\mu_{\rho_n f}})}) \preccurlyeq \rho_n^{1/2} \bm z ) \\
    &< P_n(0 \leq || \sigma_c(\bA_{\mu_{\rho_n f}}) - \E{\sigma_c(\bA_{\mu_{\rho_n f}})} ||_2 \leq ||\sqrt{\rho_n} \bm z ||_2)
  \end{align}
  and since $n$ is sufficiently large that $||\rho_n^{1/2} \bm z||_2 < \epsilon$ this implies that the probability measure of the set of graphs that satisfy $||\sigma_c(\bA) - \E{\sigma_c(\bA)}||_2 < \epsilon$ is strictly positive. Thus there exists a graph $G' \in \cG$ with adjacency matrix $\bA'$ that satisfies $||\sigma_c(\bA') - \E{\sigma_c(\bA)}||_2 < \epsilon.$ \qed
\end{lproof}
It should also be noted that the constant $C$ can be made arbitrarily close to $1 - 2\epsilon$ and so the probability of observing a graph such that $||\sigma_c(\bA') - \E{\sigma_c(\bA)}||_2 < \epsilon$ can be made arbitrarily large so long as $n$ is also taken to be sufficiently large. This is the theoretical support for Remark \ref{rem:SmallN} though in practice we are not free to choose $n$, the size of our graphs.

We are now in a position to prove Theorem \ref{thm:AlmostSure} which we restate for convenience below. Let $\lbrace \Gk \rbrace_{k=1}^N$ be an iid sample of graphs drawn from $\mu_{\rho_n f}$ where $f$ is a canonical stochastic block model kernel. Let $\bA^{(k)}$ be the adjacency matrix of graph $\Gk$ and let $\blamb^{(k)} = \sigma_c(\bA^{(k)})$. Let $G_{N}^*$ be the sample Fr\'echet mean graph with adjacency matrix $\bA_{N}^*$.
\begin{theorem}[Theorem \ref{thm:AlmostSure} in the main paper]
  \label{thm:AlmostSure-App}
  $\forall \epsilon > 0$, $\exists n^* \in \N$ such that for all $n > n^*$,
  \begin{equation}
    \lim_{N \to \infty} || \sigma_c(\bA_{N}^*) - \E{\sigma_c(\bA_{\mu_{\rho_n f}})}||_2 < \epsilon \quad a.s. \label{eqn:OnlyEigs}
  \end{equation}
\end{theorem}
As will become clear in our proof, it will be easier to work with the eigenvalues directly rather than the graphs themselves. To do so we make the following definition.
\begin{Definition}[The set of $c$ largest realizable eigenvalues of adjacency matrices of graphs in $\cG$]
  \label{def:RealizableEigs}
  \begin{equation}
    \Lambda_n^c = \left\lbrace \blamb \in \R^c \vert \exists G \in \cG \text{ with adjacency matrix } \bA \text{ such that } \blamb = \sigma_c(\bA) \right \rbrace.
  \end{equation}
\end{Definition} 

\begin{proof}
  We make two observations. First that $\Lambda_n^c \subset \R^c$ and second, equation (\ref{eqn:OnlyEigs}) depends only on the eigenvalues of $\bA^{(k)}$. These two observations allow us to recast the sample Fr\'echet mean problem over $\Lambda_n^c$ and consider the relaxed problem over $\R^c$ whose solution we show to be the optimal eigenvalues of the adjacency matrix of the sample Fr\'echet mean. 

  We first recast the problem of computing $G_N^*$ over $\Lambda_n^c$ as follows. Let $\blamb_N^* = \sigma_c(\bA_N^*)$. Then
  \begin{equation}
    \blamb_{N}^* = \underset{\blamb \in \Lambda_n^c}{\text{argmin }}\frac 1 N \sum_{k=1}^N ||\blamb^{(k)} - \blamb||_2^2. \label{eqn:EFM-L}
  \end{equation}
  We also consider the relaxed version of (\ref{eqn:EFM-L}) where the solution is in $\R^c$ instead of $\Lambda_n^c$. The relaxed problem is a trivial quadratic optimization problem with a unique solution given by
  \begin{align}
    \blamb_{N,r}^* &= \underset{\blamb \in \R^c}{\text{argmin }}\frac 1 N \sum_{k=1}^N ||\blamb^{(k)} - \blamb||_2^2. \label{eqn:EFM-L-R}\\
    &= \frac{1}{N}\sum_{k=1}^N \blamb^{(k)}
  \end{align}
  which is the classic geometric average of the observations $\blamb^{(k)}$. Now, the sample mean, $\blamb_{N,r}^*$, satisfies 
  \begin{equation}
    \frac 1 N \sum_{k=1}^N || \blamb^{(k)} - \blamb||_2^2 = ||\blamb - \blamb_{N,r}^*||_2^2 + \frac 1 N \sum_{k=1}^N ||\blamb^{(k)}||_2^2 - || \blamb_{N,r}^*||_2^2. 
  \end{equation}
  Hence, minimizing $|| \blamb - \blamb_{N,r}^*||_2^2$ is equivalent to minimizing $\frac{1}{N} \sum_{k=1}^N ||\blamb - \blamb^{(k)}||_2^2$ irrespective of the domain over which the function is minimized. This shows that an equivalent formulation of the sample Fr\'echet mean on the space of realizable eigenvalues is
  \begin{align}
    \blamb_{N}^* &= \underset{\blamb \in \Lambda_n^c}{\text{argmin }}\frac 1 N \sum_{k=1}^N ||\blamb^{(k)} - \blamb||_2^2\\
    &= \underset{\blamb \in \Lambda_n^c}{\text{argmin }} ||\blamb - \blamb_{N,r}^*||_2^2. \label{eqn:EqFM}
  \end{align} 
  Equation (\ref{eqn:EqFM}) states that we must find the realizable eigenvalues which are closest to the geometric average, the solution to the relaxed problem. Using this formulation of the sample Fr\'echet mean problem we will show that the eigenvalues of $\bA_N^*$ converge almost surely to $\E{\sigma_c(\bA_{\mu_{\rho_n f}})}$. 

  By the strong law of large number, $\forall n$, 
  \begin{equation}
    \lim_{N \to \infty} ||\blamb_{N,r}^* - \E{\sigma_c(\bA_{\mu_{\rho_n f}})}|| = 0 \quad a.s. \label{eqn:Close}
  \end{equation}
  Define
  \begin{align}
    \blamb^* &= \lim_{N \to \infty} \blamb_{N}^*\\
    &= \lim_{N \to \infty} \argmin{\blamb \in \Lambda_{n}^c} ||\blamb - \blamb_{N,r}^*||_2^2.
  \end{align}
  Since the projection on $\Lambda_n^c$ is a continuous operator,
  \begin{align}
    \lim_{N \to \infty} \argmin{\blamb \in \Lambda_{n}^c} ||\blamb - \blamb_{N,r}^*||_2^2 = \argmin{\blamb \in \Lambda_{n}^c} \lim_{N \to \infty} ||\blamb - \blamb_{N,r}^*||_2^2.
  \end{align}
  Since the norm is continuous,
  \begin{align}
    \argmin{\blamb \in \Lambda_{n}^c} \lim_{N \to \infty} ||\blamb - \blamb_{N,r}^*||_2^2 = \argmin{\blamb \in \Lambda_{n}^c}||\blamb - \lim_{N \to \infty} \blamb_{N,r}^*||_2^2.
  \end{align}
  Finally, by (\ref{eqn:Close}),
  \begin{align}
    \argmin{\blamb \in \Lambda_{n}^c}||\blamb - \lim_{N \to \infty} \blamb_{N,r}^*||_2^2 = \argmin{\blamb \in \Lambda_{n}^c}||\blamb - \E{\sigma_c(\bA_{\mu_{\rho_n f}})}||_2^2 \quad a.s. \label{eqn:MinProc}
  \end{align}
  By Lemma \ref{lem:GraphNearExpEigs}, $\forall \epsilon > 0$, there exists $n_1 \in \N$ such that for all $n > n_1$, there exists $G' \in \cG$ with adjacency matrix $\bA'$ such that
  \begin{equation}
    ||\sigma_c(\bA') - \E{\sigma_c(\bA_{\mu_{\rho_n f}})}||_2 \leq \epsilon. \label{eqn:UpperBound}
  \end{equation}
  Because $\sigma_c(\bA') \in \Lambda_n^c$, the minimizer, $\blamb^*$, in (\ref{eqn:MinProc}) satisfies
  \begin{align}
    ||\blamb^* - \E{\sigma_c(\bA_{\mu_{\rho_n f}})}||_2 \leq \epsilon \quad a.s.
  \end{align}
  Since $\blamb^* = \lim_{N \to \infty}\blamb_N^* = \lim_{N \to \infty} \sigma_c(\bA_N^*)$ this concludes our proof.  \qed
\end{proof}

\noindent \textbf{Step 4: Compiling the prior steps into a proof} \hfill \\

All that is left is to compile the results of the prior three steps into a proof for Theorem \ref{thm:SpecSimLargeGraphs} which we restate below. Let $G \in \cG$ with adjacency matrix $\bA$ such that $n^{-2/3} \ll \rho_n \ll 1$. Assume that
\begin{equation}
  0 \preccurlyeq \sigma_c(\bA)
\end{equation}
and for every $1 \leq i \neq j \leq c$, $\lambda_i \neq \lambda_j$.
\begin{theorem}[Theorem \ref{thm:SpecSimLargeGraphs} from the main paper]
  \label{thm:SpecSimLargeGraphs-App}
  $\forall \epsilon > 0$, $\exists n_1 \in \N$ such that $\forall n > n_1$, $\exists f(x,y; \bp, \bQ, \bs)$ a canonical stochastic block model kernel with $c$ communities such that
  \begin{equation}
    \lim_{N \to \infty} d_{A_c}(G,G_N^*) < \epsilon \quad a.s. \label{eqn:Orig}
  \end{equation}
  where $G_N^*$ denotes the sample Fr\'echet mean of $\lbrace \Gk \rbrace_{k=1}^N$, an iid sample distributed according to $\mu_{\rho_n f}$.
\end{theorem}
\begin{proof}
  We begin our proof by expanding the left hand side of equation (\ref{eqn:Orig}).
  \begin{align}
    d_{A_c}(G,G_N^*) &= ||\sigma_c(\bA) - \sigma_c(\bA_N^*)||_2\\
    &= ||\sigma_c(\bA) - \E{\sigma_c(\bA_{\mu_{\rho_n f}})} + \E{\sigma_c(\bA_{\mu_{\rho_n f}})} - \sigma_c(\bA_N^*)||_2. \label{eqn:Reform}
  \end{align}
  Let $\epsilon > 0$. We will show that there exists $n^* \in \N$ such that for all $n>n^*$, there exists a stochastic block model kernel probability measure $f(x,y; \bp, \bQ, \bs)$ such that the following two inequalities hold, 
  \begin{align}
    &|| \sigma_c(\bA) - \E{\sigma_c(\bA_{\mu_{\rho_n f}})}||_2 < \frac \epsilon 2 ,  \label{eqn:First}\\
    &||\E{\sigma_c(\bA_{\mu_{\rho_n f}})} - \sigma_c(\bA_N^*)||_2 < \frac \epsilon 2 \quad a.s.  \label{eqn:Second}
  \end{align}
  We begin with (\ref{eqn:First}). Define $L_f$ as in Lemma \ref{lem:SBMCorrectEigs} taking $\bm{\theta} = \frac{\sigma_c(\bA)}{n \rho_n}$. Then $\lambda_i(\bA) = n \rho_n \lambda_i(L_f)$. By Theorem \ref{thm:EstLargeEigs},
  \begin{equation}
    \E{\sigma_c(\bA_{\mu_{\rho_n f}})} = n \rho_n \bm{\theta} + \mathcal{O}(\sqrt{\rho_n}).
  \end{equation}
  Since $\sqrt{\rho_n} \to 0$, there exists an $n_1 \in \N$ such that for all $n > n_1$,
  \begin{equation}
    || n \rho_n \bm{\theta} - \E{\sigma_c(\bA_{\mu_{\rho_n f}})}||_2 < \frac \epsilon 2 .
  \end{equation}
  Theorem \ref{thm:AlmostSure} implies the existence of an $n_2 \in \N$ such that inequality (\ref{eqn:Second}) holds. Taking $n^* = \max(n_1,n_2)$ implies that both inequalities hold and concludes the proof of Theorem \ref{thm:SpecSimLargeGraphs}. \qed
\end{proof}

\section{Proof of Theorem \ref{thm:GeomEigsFM}}
\label{app:GeomEigsFM}
The proof of Theorem \ref{thm:GeomEigsFM}, which we restate below, is a consequence of Lemmas \ref{lem:SBMCorrectEigs} and \ref{lem:GraphNearExpEigs} along with Theorem \ref{thm:EstLargeEigs}. Let $\lbrace \Gk \rbrace_{k=1}^N$ have sample Fr\'echet mean $G_N^*$.
\begin{theorem}[Theorem \ref{thm:GeomEigsFM} in the main paper]
  \label{thm:GeomEigsFM-app}
  $\forall \epsilon > 0$, $\exists n^* \in \N$ such that $\forall n > n^*$,
  \begin{equation}
    \left \|  \sigma_c(\bA_N^*) - \frac 1 N \sum_{k=1}^N \sigma_c(\bA^{(k)})  \right \|_2 < \epsilon.
  \end{equation}
\end{theorem}
\begin{proof}
  Let $\blamb^{(k)} = \sigma_c(\bA^{(k)})$. Take $\bm{\theta} = \frac{1}{n \rho_n N} \sum_{k=1}^N \blamb^{(k)}$ in Lemma \ref{lem:SBMCorrectEigs}. Let $\epsilon > 0$. By Lemma \ref{lem:GraphNearExpEigs} there exists $n_1 \in \N$ such that for all $n > n_1$, there exists a graph $G' \in \cG$ with adjacency matrix $\bA'$such that
  \begin{equation}
    ||\sigma_c(\bA') - \E{\sigma_c(\bA_{\mu_{\rho_n f}})}||_2 < \epsilon.
  \end{equation}
  By Theorem \ref{thm:EstLargeEigs}, there exists $n_2 \in \N$ such that for all $n > n_2$, 
  \begin{align}
    \forall i = 1,...,c, \left| n \rho_n \lambda_i(L_f) - \E{\lambda_i(\bA_{\mu_{\rho_n f}})} \right| < \epsilon.
  \end{align}
  Since $n \rho_n \lambda_i(L_f) = \frac 1 N \sum_{k=1}^N \lambda_i(\bA^{(k)})$, there exists a graph $G'$ with adjacency matrix $\bA'$ that satisfies
  \begin{equation}
    ||\sigma_c(\bA') - \frac 1 N \sum_{k=1}^N \blamb^{(k)}||_2 < C\epsilon,
  \end{equation}
  where $C$ is an arbitrary positive constant. We recall, (see (\ref{eqn:EqFM}), that the spectrum of the sample Fr\'echet mean is the solution to 
  \begin{align}
    \blamb_{N}^* = \underset{\blamb \in \Lambda_n^c}{\text{argmin }} ||\blamb - \frac{1}{N}\sum_{k=1}^N \blamb^{(k)}||_2^2
  \end{align}
  Now, take $n^* = \max(n_1,n_2)$, then the existence of the graph $G'$ with adjacency matrix $\bA'$ shows that
  \begin{equation}
    ||\sigma_c(\bA_N^*) - \frac{1}{N}\sum_{k=1}^N \blamb^{(k)}||_2^2 < ||\sigma_c(\bA') - \frac{1}{N}\sum_{k=1}^N \blamb^{(k)}||_2^2 < \epsilon.
  \end{equation}
  Thus the adjacency matrix of the sample Fr\'echet mean must have eigenvalues that are within $\epsilon$ of the geometric average. \qed
\end{proof}

\section{Proof of Theorem \ref{thm:HighProbStat}}
\label{app:HighProbStat}
Let $\lbrace \tilde G^{(k)} \rbrace_{k = 1}^{\tilde N}$ be a sample of graphs distributed according to $\mu_{\rho_n f}$ where $f$ is the canonical stochastic block model kernel. Define the set mean graph by
\begin{equation}
  \widehat{G}_{\tilde N,\mu_{\rho_n f}}^* = \argmin{\tilde G \in \lbrace \tilde{G}^{(k)} \rbrace_{k = 1}^{\tilde N}} \frac{1}{\tilde N} \sum_{k = 1}^{\tilde N} d_{A_c}^2(\tilde G,\tilde{G}^{(k)})
\end{equation}
with adjacency matrix $\hat{\bA}_{N,\mu_{\rho_n f}}^*$.

\begin{theorem}
  $\forall \epsilon > 0$, 
  \begin{equation}
    \lim_{n \to \infty} P(||\sigma_c(\hat{\bA}_{\tilde N, \mu_{\rho_n f}}^*) - \E{\sigma_c(\bA_{\mu_{\rho_n f}})}||_2 > \epsilon) = 0.
  \end{equation}
\end{theorem}
\begin{proof}
  We first prove that $\sigma_c(\bA_{\mu_{\rho_n f}})$ converges in probability to $\E{\sigma_c(\bA_{\mu_{\rho_n f}})}$ for large graph size. From Theorem \ref{thm:ConvDist},
  \begin{equation}
    \frac{1}{\sqrt{\rho_n}} \left( \sigma_c(\bA_{\mu_{\rho_n f}}) - \E{\sigma_c(\bA_{\mu_{\rho_n f}})} \right) \overset{d}{\to} Z \sim N(0, \Sigma) \label{eqn:ConvDistHighProb}
  \end{equation}
  where the $c \times c$ covariance matrix $\Sigma$ is given by (\ref{eqn:CovMat}). Let $z > 0$, $\epsilon > 0$, and $\eta > 0$. Because of (\ref{eqn:ConvDistHighProb}), $\exists n_0 \in \N$, $\forall n > n_0$,
  \begin{equation}
    \left| P\left( \frac{1}{\sqrt{\rho_n}} |\lambda_i(\bA_{\mu_{\rho_n f}}) - \E{\lambda_i(\bA_{\mu_{\rho_n f}})}| \leq z \right) - P\left( |z_i| \leq z; i = 1,...,c \right)\right| \leq \frac \eta 2.
  \end{equation}
  Thus
  \begin{equation}
    P\left( |z_i| \leq z; i = 1,...,c \right) - \frac \eta 2 \leq  P\left( \frac{1}{\sqrt{\rho_n}} |\lambda_i(\bA_{\mu_{\rho_n f}}) - \E{\lambda_i(\bA_{\mu_{\rho_n f}})}| \leq z \right).
  \end{equation}
  Now $P(|| \bm z || \leq z) \leq P(|z_i| \leq z; i = 1,...,c)$ and
  \begin{equation}
    P\left( \frac{1}{\sqrt{\rho_n}} |\lambda_i(\bA_{\mu_{\rho_n f}}) - \E{\lambda_i(\bA_{\mu_{\rho_n f}})}| \leq z \right) \leq P\left(||\sigma_c(\bA_{\mu_{\rho_n f}}) - \E{\sigma_c(\bA_{\mu_{\rho_n f}})}||_2 \leq \sqrt{c \rho_n} z \right).
  \end{equation}
  Because $\lim_{z \to \infty} P(||\bm z|| \leq z) = 1$, $\exists z_0 > 0$ such that
  \begin{equation}
    1 - \frac \eta 2 < P(|| \bm z || \leq z_0).
  \end{equation}
  Also, $\lim_{n \to \infty} \rho_n = 0$ so $\exists n_1 \in \N$ such that $\forall n > n_1$, $\sqrt{\rho_n} < \frac{\epsilon}{z_0 \sqrt c}$ or $\sqrt{c \rho_n z_0} < \epsilon.$  In summary, $\forall \epsilon > 0$, $\forall \eta > 0$, $\exists n_2 = \max(n_0, n_1)$ where
  \begin{equation}
    1 - \eta < P\left(||\sigma_c(\bA_{\mu_{\rho_n f}}) - \E{\sigma_c(\bA_{\mu_{\rho_n f}})}||_2 < \epsilon \right).
  \end{equation}
  Equivalently, 
  \begin{equation}
    P\left(||\sigma_c(\bA_{\mu_{\rho_n f}}) - \E{\sigma_c(\bA_{\mu_{\rho_n f}})}||_2 \geq \epsilon \right) < \eta.
  \end{equation}
  In other words, $\forall \epsilon > 0$, 
  \begin{equation}
    \lim_{n \to \infty} P\left(||\sigma_c(\bA_{\mu_{\rho_n f}}) - \E{\sigma_c(\bA_{\mu_{\rho_n f}})}||_2 \geq \epsilon \right) = 0. \label{eqn:ConvProb}
  \end{equation}
  We now show that the largest $c$ eigenvalues of the adjacency matrix of the set Fr\'echet mean graph converges in probability to $\E{\sigma_c(\bA_{\mu_{\rho_n f}})}$. Let $\epsilon > 0$ and let $\eta > 0$. Let $\tilde N > 0$ and consider the event
  \begin{equation}
    \mathcal{E} = \lbrace \bA^{(1)},...,\bA^{(\tilde N)} ; ||\sigma_c(\widehat{\bA}_{\tilde N, \mu_{\rho_n f} }) - \E{\sigma_c(\bA_{\mu_{\rho_n f}})} ||_2 > \epsilon \rbrace.
  \end{equation}
  Note that $\exists k_0 \in \lbrace 1,..., \tilde N \rbrace$ where $\widehat{\bA}_{\tilde N, \mu_{\rho_n f} } = \bA^{(k_0)}$ where $\bA^{(k_0)} \sim \mu_{\rho_n f}$. Now because of (\ref{eqn:ConvProb}), $\exists n_0 \in \N$ where $\forall n > n_0$, $P(\mathcal{E}) < \eta$. We conclude that $\forall \epsilon > 0$, 
  \begin{equation}
    \lim_{n \to \infty} P\left(||\sigma_c(\widehat{\bA}_{\tilde N, \mu_{\rho_n f} }) - \E{\sigma_c(\bA_{\mu_{\rho_n f}})} ||_2 > \epsilon \right) = 0.
  \end{equation}
\end{proof}
\end{document}